\documentclass[11pt]{article}
\usepackage{amsmath,amsthm,amssymb,amsfonts}
\usepackage{mathrsfs}
\usepackage{mathtools}
\usepackage{fullpage}
\usepackage{graphicx}
\usepackage{algorithm}
\usepackage{algpseudocode}
\usepackage{tikz}
\usepackage{natbib}
\usetikzlibrary{arrows.meta, calc, positioning, shapes.geometric}

\tikzset{
  box/.style={rectangle, draw=black, thick, rounded corners, minimum height=1cm, align=center},
  arrow/.style={-Stealth, thick}
}
\usepackage{enumitem}
\usepackage{bm}
\usepackage[hidelinks]{hyperref}
\newtheorem{theorem}{Theorem}[section]

\newtheorem{proposition}[theorem]{Proposition}
\newtheorem{corollary}[theorem]{Corollary}
\newtheorem{definition}[theorem]{Definition}

\newcommand{\R}{\mathbb{R}}
\newcommand{\cL}{\mathcal{L}}
\newcommand{\cN}{\mathcal{N}}
\newcommand{\E}{\mathbb{E}}

\begin{document}
\date{}
\title{Symplectic Generative Networks (SGNs): A Hamiltonian Framework for Invertible Deep Generative Modeling}
\vspace{1ex}
\author{Agnideep Aich${ }^{1}$\thanks{Corresponding author: Agnideep Aich, \texttt{agnideep.aich1@louisiana.edu}, ORCID: \href{https://orcid.org/0000-0003-4432-1140}{0000-0003-4432-1140}}
 \hspace{0pt}, Ashit Baran Aich ${ }^{2}$ \hspace{0pt}
\\ ${ }^{1}$ Department of Mathematics, University of Louisiana at Lafayette, \\ Lafayette, Louisiana, USA \\  ${ }^{2}$ Department of Statistics, formerly of Presidency College, \\ Kolkata, India \\}
\maketitle
\vspace{-20pt}

\begin{abstract}
 We introduce the \emph{Symplectic Generative Network (SGN)}, a deep generative model that leverages Hamiltonian mechanics to construct an invertible, volume-preserving mapping between a latent space and the data space. By endowing the latent space with a symplectic structure and modeling data generation as the time evolution of a Hamiltonian system, SGN achieves exact likelihood evaluation without incurring the computational overhead of Jacobian determinant calculations. In this work, we provide a rigorous mathematical foundation for SGNs through a comprehensive theoretical framework that includes: (i) complete proofs of invertibility and volume preservation, (ii) a formal complexity analysis with theoretical comparisons to Variational Autoencoders and Normalizing Flows, (iii) strengthened universal approximation results with quantitative error bounds, (iv) an information-theoretic analysis based on the geometry of statistical manifolds, and (v) an extensive stability analysis with adaptive integration guarantees. These contributions highlight the fundamental advantages of SGNs and establish a solid foundation for future empirical investigations and applications to complex, high-dimensional data.

\end{abstract}



\section{Introduction}
\label{sec:intro}

Evaluating exact likelihood in deep generative models involves a trade-off. Normalizing Flows (NFs) \citep{RezendeMohamed2015} reach this by applying a series of invertible transformations, but each step requires calculating the log-determinant of the Jacobian ($\log|\det J|$), which can become computationally expensive as data dimensionality increases. This often restricts how deep or complex the model can be. In contrast, Variational Autoencoders (VAEs) \citep{Kingma2013} avoid this cost by optimizing a variational lower bound (ELBO), but this approach only approximates the likelihood, introducing a gap.

Recent research, including Continuous Normalizing Flows (CNFs) \citep{Chen2019Residual}, approaches the problem using neural ODEs. This method replaces the determinant of a $D \times D$ matrix with the trace of the Jacobian, which can be estimated using stochastic methods. Although this is an innovative step, it introduces new challenges, including numerical-solver error, stability issues, and increased estimator variance. This leads to a key question: do we have to choose between the high computational cost of NFs, the approximation gap of VAEs, or the numerical challenges of CNFs?

In this paper, we present Symplectic Generative Networks (SGNs), a framework that addresses this challenge using ideas from Hamiltonian mechanics. Instead of treating the latent transformation as a generic ODE, we model it as the time evolution of a Hamiltonian system.

When we give the latent space a canonical symplectic structure, meaning it has positions $q$ and momenta $p$, and set the dynamics using a neural Hamiltonian $H_\psi$, the flow becomes symplectic. According to Liouville's theorem, this means the map preserves volume exactly. As a result, the Jacobian determinant of the latent transport is always one, so its logarithm is always zero.

This cost-free latent transport is the stable, volume-preserving foundation of our framework. We use it in two different training approaches. The SGN core remains the same in both variants. The only difference is how the phase space $\mathcal{Z}\subset\mathbb{R}^{2d}$ maps to the data space $\mathcal{X}\subset\mathbb{R}^{D}$.

\begin{enumerate}

  \item \textbf{SGN-Flow (The Invertible, Exact-Likelihood Model):}
This variant is a pure, invertible normalizing flow. We compose the zero-cost symplectic flow $\Phi_T: \mathcal{Z} \to \mathcal{Z}$ with a final (typically simple) invertible mapping $g_\theta: \mathcal{Z} \to \mathcal{X}$. The total log-likelihood is:
\[
\log p(x) = \log p_0(z_0) + \log \left|\det D (\Phi_T^{-1})(z_T)\right| + \log \left|\det D (g_\theta^{-1})(x)\right|
\]
Since the flow $\Phi_T$ is symplectic, its $\log|\det J|$ is zero. The entire cost reduces to the log-determinant of the terminal map $g_\theta$ alone, which can be designed to be trivial (e.g., an orthogonal map) or low-cost (e.g., a simple coupling layer).

  \item \textbf{SGN-VAE (The Hybrid, ELBO-Based Model):}This variant provides the flexibility of a stochastic decoder $p_\theta(x \mid z_T)$ and an encoder $q_\phi(z_0 \mid x)$, just like a standard VAE. However, the ELBO (Eq.~\ref{eq:elbo}) simplifies significantly:
\[
\mathcal{L}(x) = \mathbb{E}_{q_\phi(z_0 \mid x)}\Bigl[\log p_\theta\bigl(x \mid \Phi_T(z_0)\bigr)\Bigr] - D_{\mathrm{KL}}\Bigl(q_\phi(z_0 \mid x) \,\|\, p_0(z_0)\Bigr)
\]
Crucially, no Jacobian correction term is needed for the transformation from $z_0$ to $z_T$. The SGN core acts as a highly structured, invertible, and volume-preserving latent transport mechanism inside the VAE, providing stable dynamics without complicating the objective.

\end{enumerate}
To keep the system stable, we break down the continuous Hamiltonian dynamics into steps using the leapfrog (Störmer-Verlet) integrator. This method is symplectic, so it preserves the unit-Jacobian property at each step and retains these guarantees at the discrete level.

This paper lays out the theoretical basis for SGNs. To clarify our contributions, we compare SGNs to existing models.

\begin{enumerate}
    
 \item \textbf{A ``Zero-Cost Jacobian'' Normalizing Flow:}
The SGN-Flow variant (Sec.~\ref{subsec:objectives}) acts as a novel NF architecture.
\begin{itemize}
    \item 

\textbf{vs.\ NFs:} Instead of stacking $K$ layers, each with a $\mathcal{O}(D^3)$ $\log|\det J|$ cost, SGNs perform $T$ integration steps, each with a $\mathcal{O}(D)$ gradient-evaluation cost, and pay zero $\log|\det J|$ cost for the flow. The only determinant cost comes from a single, simple terminal map.

\item  \textbf{vs.\ CNFs:} Instead of estimating the log-trace of a generic vector field, SGNs guarantee the log-determinant is zero by construction. This removes the need for stochastic trace estimators and their associated variance and stability concerns.
\end{itemize}

\item  \textbf{A Structure-Preserving VAE:}

The SGN-VAE variant (Sec.~\ref{subsec:objectives}.B) acts as a hybrid VAE.

\begin{itemize}
    \item  \textbf{vs.\ VAEs:} Standard VAEs use a simple Gaussian prior. VAEs with latent flows (e.g., \citep{RezendeMohamed2015}) must pay the full Jacobian cost inside the ELBO. SGN-VAE provides complex, structured latent transport with no additional objective-function complexity, as the Jacobian term vanishes.

\end{itemize}
\end{enumerate}
\textbf{A Rigorous Theoretical Foundation:}

We provide a comprehensive theoretical analysis that was absent in prior conceptual work. This includes:
\begin{enumerate}
    \item 

\textbf{Complexity Analysis (Sec.~\ref{sec:theoretical_comparison}):} Formal proofs of SGN's $\mathcal{O}(T \cdot d)$ computational advantage over the $\mathcal{O}(K \cdot C_J(d))$ cost of NFs.

\item \textbf{Universal Approximation (Sec.~\ref{sec:universal_approximation}):} Strengthened proofs (Theorems~\ref{thm:univ_approx_stronger}, \ref{thm:quant_bounds}) showing SGNs can universally approximate any volume-preserving diffeomorphism.

\item  \textbf{Information Theory (Sec.~\ref{sec:information_theory}):} An information-geometric analysis (Theorem~\ref{thm:info_sgn}) linking Hamiltonian dynamics to geodesic flows on statistical manifolds.

\item \textbf{Stability Analysis (Sec.~\ref{sec:expanded_stability}):} A complete stability hierarchy (Theorem~\ref{thm:stability_hierarchy}) with adaptive integration guarantees (Theorem~\ref{thm:adaptive}) and rigorous bounds for neural network Hamiltonians (Theorem~\ref{thm:stability_domains}).

\end{enumerate}

This paper aims to build a solid theoretical foundation for SGNs, which is an important step before conducting thorough empirical tests. Section~\ref{sec:related} covers related work in generative modeling. Section~\ref{sec:model} explains the phase-space setup, the leapfrog integrator, and the two SGN training objectives. Sections~\ref{sec:theory} through~\ref{sec:expanded_stability} discuss the main theoretical results, including complexity, approximation, information theory, and numerical stability. Section~\ref{sec:training} describes the unified training algorithm for both the SGN-Flow and SGN-VAE approaches. The paper ends in Section~\ref{sec:conclusion_future} with conclusions and suggestions for future empirical studies.

\section{Related Work}\label{sec:related}

Symplectic Generative Networks (SGNs) bring together concepts from invertible deep learning, physics-informed neural networks, and variational inference. To highlight our contribution, we compare the SGN-Flow and SGN-VAE models with leading methods in each area.

\subsection{SGN-Flow vs.\ Invertible Likelihood Models}

This approach focuses on evaluating the exact likelihood, $p(x) = p_0(z_0) \left| \det J_{f^{-1}}(x) \right|$. The main difficulty lies in managing the cost and stability of the Jacobian determinant.

\textbf{Normalizing Flows (NFs):} Discrete NFs \citep{RezendeMohamed2015} build the invertible map $f$ by stacking $K$ layers, so $f = f_K \circ \dots \circ f_1$. The total log-determinant is the sum $\sum_i \log \left|\det J_{f_i}\right|$. This setup creates a trade-off: simple triangular maps have $\mathcal{O}(D)$ cost but less flexibility, while dense maps have $\mathcal{O}(D^3)$ cost, which becomes impractical as $D$ increases.

\textbf{Our Advantage:} The SGN-Flow variant, described in Section~\ref{subsec:objectives}.A, introduces a new normalizing flow architecture that ensures the flow's $\log|\det J|$ term is always zero. According to Theorem~\ref{thm:complexity}, the computational cost is $\mathcal{O}(T \cdot d)$, based on $T$ gradient steps, and does not depend on calculating any determinants for the flow itself. The only time a determinant is needed is for a single, straightforward terminal map $g_\theta$.

\textbf{Continuous Normalizing Flows (CNFs):}CNFs or Neural ODEs \citep{Chen2019Residual} approach the problem using a continuous-time flow, $\dot{z} = v(z, t)$. This method replaces the determinant with the trace of the Jacobian, $\log p(x(T)) = \log p_0(z(0)) - \int_{0}^{T} \mathrm{tr}\left( D v(z(t)) \right) dt$. The trace is typically estimated stochastically, for example with Hutchinson's estimator, which can lead to higher variance and numerical errors from the ODE solver.

\textbf{Our Advantage:} SGNs form a specific, well-organized type of CNF. When we require the vector field $v$ to be Hamiltonian, meaning $v = J \nabla H_\psi$, we do more than just estimate the trace—we ensure it is exactly zero. In fact, the divergence (or trace) of any Hamiltonian vector field is always zero:
\[
\mathrm{tr}(D v) = \nabla \cdot v = \sum_i \left( \frac{\partial \dot{q}_i}{\partial q_i} + \frac{\partial \dot{p}_i}{\partial p_i} \right) = \sum_i \left( \frac{\partial^2 H_\psi}{\partial q_i \partial p_i} - \frac{\partial^2 H_\psi}{\partial p_i \partial q_i} \right) \equiv 0
\]

(This holds exactly for the continuous flow generated by a smooth $H_\psi$. Discretization preserves volume exactly only when the numerical integrator, like leapfrog, is itself symplectic.)
SGNs use a stable, symplectic integrator instead of a general neural ODE. This approach keeps the zero-divergence property intact at the discrete level and removes the need for trace estimation.

\subsection{SGN-VAE vs.\ VAEs with Latent Dynamics}

In this field, the ELBO ($\mathcal{L} \le \log p(x)$) is used. The main challenge is balancing the ELBO's approximation gap with the simplicity of the prior, $p(z) = \mathcal{N}(0, I)$.

\textbf{Standard VAEs:} The VAE \citep{Kingma2013} objective is computationally efficient. However, using a simple Gaussian prior can create an information bottleneck, which may lead the model to learn a less effective latent representation.

\textbf{VAEs with Latent Flows:} To solve this problem, many studies (e.g., \citep{RezendeMohamed2015}) use a normalizing flow $f$ to turn the simple prior into a more flexible distribution, $z = f(z_0)$. But this approach brings back the full Jacobian cost in the ELBO, since it now needs to include the flow's volume change: $\log p(z) = \log p_0(z_0) - \log |\det D f(z_0)|$.

\textbf{Our Advantage:} The SGN-VAE variant (Sec.~\ref{subsec:objectives}.B) enables complex and structured latent transport, but keeps the objective straightforward. Since the symplectic flow $\Phi_T$ preserves volume, it changes the prior $p_0(z_0)$ into a complex, multi-modal distribution $p(z_T)$, and the $\log|\det J|$ term remains zero. As a result, SGN-VAE combines the expressive power of latent flows with the simplicity and efficiency of a standard VAE.

\subsection{Physics-Informed and Structured Generative Models}

Our research is part of a larger movement to bring strong mathematical and physical foundations to deep learning.

\textbf{Physics-Informed Models:} Hamiltonian Neural Networks (HNNs) \citep{Greydanus2019} showed that neural networks can learn Hamiltonians from data and conserve energy when predicting physical systems. SGNs expand on this by using the HNN as the core of a generative likelihood model, not just for predicting dynamics.

\textbf{Reversible Architectures:} RevNets \citep{Gomez2017} showed that reversible architectures can save memory during backpropagation. Since the symplectic integrator in SGNs is also reversible, it offers the same advantage and makes it possible to train deep flows without increasing memory use.

\textbf{Structured Generative Classifiers:} The idea of building models on strong theoretical foundations is not limited to physics. For instance, the Deep Copula Classifier (DCC) \citep{aich2025dcc}
is a class-conditional generative model built on the foundation of 
copula theory \citep{sklar1959,Nelsen2006}. Instead of 
assuming feature independence (like Naive Bayes \citep{McCallumNigam1998}), 
DCC explicitly ``separates marginal estimation from dependence modeling 
using neural copula densities'' \citep{aich2025dcc}. Similar to how SGNs use symplectic geometry for stable and interpretable latent transport, DCC uses copula theory to create a provably
Bayes-consistent \citep{aich2025dcc} and 
interpretable classifier \citep{aich2025dcc} 
that directly models feature dependencies.

Overall, our work provides a novel synthesis. By grounding our generative model in symplectic geometry, SGNs offer a unique and compelling set of trade-offs: the exact-likelihood of NFs, the zero-cost latent transport of VAEs, and the theoretical stability of physics-informed models.
\section{Symplectic Generative Networks (SGNs)}
\label{sec:model}

SGNs realize the latent-to-data transformation as a Hamiltonian flow on a $2d$-dimensional phase space endowed with the canonical symplectic form. We present (i) the phase-space setup and Hamiltonian dynamics, (ii) the symplectic time-discretization used in practice, and (iii) two training regimes with precise likelihood objectives: a fully invertible \emph{SGN-Flow} (exact log-likelihood) and a hybrid \emph{SGN-VAE} (ELBO). We also state minimal regularity assumptions needed for well-posedness and stable training.

\subsection{Phase Space, Prior, and Hamiltonian Dynamics}
\label{subsec:phase_ham}
Let the latent phase space be $\mathcal Z=\R^{2d}$ with canonical coordinates
\[
z=(q,p),\qquad q,p\in\R^d,
\]
and canonical symplectic matrix $J=\begin{psmallmatrix} 0&I_d\\-I_d&0\end{psmallmatrix}$. We equip $\mathcal Z$ with the standard Gaussian prior
\[
p_0(z)=\cN(0,I_{2d}).
\]
A neural Hamiltonian $H_\psi:\R^{2d}\to\R$ (parameters $\psi$) induces the Hamiltonian vector field
\begin{equation} \label{eq:hamilton}
\dot z = J\nabla H_\psi(z)\quad\Longleftrightarrow\quad
\dot q=\nabla_p H_\psi(q,p),\qquad \dot p=-\nabla_q H_\psi(q,p).
\end{equation}
For a fixed horizon $T>0$, let $\Phi_T:\R^{2d}\to\R^{2d}$ denote the time-$T$ flow map. Under $H_\psi\in C^1$ with locally Lipschitz gradient, the flow exists and is a $C^1$-diffeomorphism. By Liouville’s theorem, $\Phi_T$ preserves the symplectic form $\omega=dq\wedge dp$ and phase-space volume:
\begin{equation}
\label{eq:unit-jac}
\det D\Phi_T(z)=1\quad\text{for all }z\in\R^{2d}.
\end{equation}

\paragraph{Generative viewpoint.}
SGNs transport the prior through $\Phi_T$ to produce a latent $Z_T=\Phi_T(Z_0)$ with $Z_0\sim p_0$, then map to data in one of two ways:
\begin{itemize}
\item \textbf{SGN-Flow (invertible):} set $x = g_\theta(z_T)$ where $g_\theta:\R^{2d}\!\to\R^{D}$ is a diffeomorphism (often $D=2d$ and $g_\theta$ is identity or an invertible, volume-\emph{changing} map with tractable $\log|\det Dg_\theta|$).
\item \textbf{SGN-VAE (decoder):} sample $x\sim p_\theta(x\mid z_T)$ from a stochastic decoder.
\end{itemize}

For the SGN-Flow variant, we typically assume the data dimension $D$ matches the phase space dimension $2d$, i.e., $g_\theta: \mathbb{R}^{2d} \to \mathbb{R}^{2d}$, often with $g_\theta$ being the identity or a simple transformation. However, the framework allows for $g_\theta: \mathbb{R}^{2d} \to \mathbb{R}^{D}$ where $D \neq 2d$, provided $g_\theta$ remains an invertible map between manifolds of potentially different dimensions (e.g., embedding a lower-dimensional manifold). For SGN-VAE, the decoder maps from $\mathbb{R}^{2d}$ to the data space $\mathbb{R}^D$ without requiring $D=2d$.

A representative \emph{phase portrait} with energy level sets and a symplectic trajectory is shown in Fig.~\ref{fig:ham_phase}.

\begin{figure}[H]
\centering
\begin{tikzpicture}[scale=0.95, every node/.style={transform shape}]
  \draw[->] (-3.6,0) -- (3.6,0) node[right] {$q$};
  \draw[->] (0,-3.6) -- (0,3.6) node[above] {$p$};
  \foreach \r in {1,2,3}{ \draw[blue!60!black] (0,0) circle (\r); }
  \draw[very thick, red, ->] (1,0) arc[start angle=0, end angle=140, radius=1];
  \node[text=blue!60!black] at (2.2,2.2) {Energy levels};
  \node[red] at (0.1,1.25) {Hamiltonian flow};
  \node[draw=gray, fill=white, align=left, text width=4.4cm] at (2.6,-1.5) {
    \textbf{Hamilton's eqs.}\\[2pt]
    $\dot q = \partial H/\partial p$\\
    $\dot p = -\partial H/\partial q$\\[4pt]
    \emph{Symplectic $\Rightarrow$ volume preserving}
  };
\end{tikzpicture}
\caption{Hamiltonian dynamics in phase space. Concentric blue curves are constant energy; the red curve shows the flow.}
\label{fig:ham_phase}
\end{figure}

\subsection{Symplectic Time Discretization}
\label{subsec:integrator}
We use the leapfrog/Stormer–Verlet scheme with step size $\Delta t$ and $N=T/\Delta t$ steps:
\begin{align}
p_{t+\frac12} &= p_t - \frac{\Delta t}{2}\,\nabla_q H_\psi(q_t,p_t), \label{eq:lf1}\\
q_{t+1} &= q_t + \Delta t\,\nabla_p H_\psi(q_t,p_{t+\frac12}), \label{eq:lf2}\\
p_{t+1} &= p_{t+\frac12} - \frac{\Delta t}{2}\,\nabla_q H_\psi(q_{t+1},p_{t+\frac12}). \label{eq:lf3}
\end{align}
Each sub-update is a shear with unit determinant; hence their composition $\Phi_T^{(\Delta t)}$ is symplectic and satisfies \eqref{eq:unit-jac} exactly at the discrete level. Local error is $O(\Delta t^3)$ and the global state error is $O(T\Delta t^2)$ for fixed $T$.

The leapfrog composition is summarized in Fig.~\ref{fig:leapfrog}.

\begin{figure}[H]
\centering
\resizebox{\textwidth}{!}{%
\begin{tikzpicture}[node distance=1.7cm]
  \node[box, text width=5.3cm] (s1) {Step 1:\quad
    $p_{t+\frac12} = p_t - \dfrac{\Delta t}{2}\,\nabla_q H(q_t,p_t)$};
  \node[box, right=1.4cm of s1, text width=5.3cm] (s2) {Step 2:\quad
    $q_{t+1} = q_t + \Delta t\,\nabla_p H(q_t,p_{t+\frac12})$};
  \node[box, right=1.4cm of s2, text width=5.3cm] (s3) {Step 3:\quad
    $p_{t+1} = p_{t+\frac12} - \dfrac{\Delta t}{2}\,\nabla_q H(q_{t+1},p_{t+\frac12})$};
  \draw[arrow] (s1) -- (s2);
  \draw[arrow] (s2) -- (s3);
  \node[draw=gray, fill=white, align=left, text width=5.5cm, below=1.0cm of s2] {
    \textbf{Properties}\\[-1pt]
    • Symplectic (volume preserving)\\
    • Reversible, 2nd order\\
    • Local error $O(\Delta t^3)$,\; global $O(\Delta t^2)$
  };
\end{tikzpicture}}
\caption{Leapfrog integration used in SGNs.}
\label{fig:leapfrog}
\end{figure}

\subsection{Likelihoods and Training Objectives}
\label{subsec:objectives}
Because $\Phi_T^{(\Delta t)}$ is volume-preserving, \emph{no Jacobian term} arises from the Hamiltonian evolution. The overall likelihood depends solely on the final data mapping.

Fig.~\ref{fig:sgn_arch} contrasts the end-to-end data path highlighting that only the final mapping contributes a log-det term.

\begin{figure}[H]
\centering
\begin{tikzpicture}[node distance=2.0cm]
  \node[box, text width=2.3cm] (data)    {Data Space\\ $x$};
  \node[box, right=1.9cm of data] (enc)  {Encoder\\ $q_\phi(z_0\mid x)$};
  \node[box, right=1.9cm of enc, text width=3cm] (ham) {Hamiltonian\\ $H_\psi(q,p)$\\ (symplectic flow $\Phi_T$)};
  \node[box, right=1.9cm of ham] (dec)   {Decoder\\ $p_\theta(x\mid z_T)$};
  \node[box, below=1.2cm of ham, text width=2.3cm] (prior) {Prior\\ $\mathcal N(0,I)$};
  \draw[arrow] (data) -- (enc);
  \draw[arrow] (enc) -- node[above] {$(q_0,p_0)$} (ham);
  \draw[arrow] (ham) -- node[above] {$z_T=\Phi_T(z_0)$} (dec);
  \draw[arrow] (prior) -- node[left] {reverse flow} (ham);
\end{tikzpicture}
\caption{SGN pipeline: encoder $\to$ symplectic flow $\to$ decoder. The flow is volume-preserving, so only the terminal mapping contributes a Jacobian term.}
\label{fig:sgn_arch}
\end{figure}

\paragraph{(A) SGN-Flow (exact log-likelihood).}
Assume $g_\theta$ is a diffeomorphism $\R^{2d}\leftrightarrow\R^{D}$ with tractable $\log|\det Dg_\theta|$. Define $f_{\psi,\theta}:=g_\theta\circ \Phi_T^{(\Delta t)}$. For $x\in\R^D$,
\[
\log p_{\psi,\theta}(x)
= \log p_0\!\big(f_{\psi,\theta}^{-1}(x)\big)
+ \log\left|\det D f_{\psi,\theta}^{-1}(x)\right|.
\]

Because $\Phi_T^{(\Delta t)}$ is symplectic (and thus its inverse is also symplectic), $\log\left|\det D (\Phi_T^{(\Delta t)})^{-1}\right|=0$. Using the chain rule, $D f_{\psi,\theta}^{-1}(x) = D (\Phi_T^{(\Delta t)})^{-1}(z_T) \cdot D g_\theta^{-1}(x)$, and the determinant property $\det(AB) = \det(A)\det(B)$, we have:

\[
\log\left|\det D f_{\psi,\theta}^{-1}(x)\right|
= \log\left|\det D g_\theta^{-1}(x)\right|
= -\,\log\left|\det D g_\theta(z_T)\right|_{z_T=g_\theta^{-1}(x)}.
\]
Thus the Hamiltonian contributes \emph{no} determinant cost; only $g_\theta$’s (generally low-cost, e.g., triangular/coupling) Jacobian is needed. Maximizing the exact likelihood over $(\psi,\theta)$ yields an invertible, fully normalizing-flow–compliant model with a symplectic core.

\paragraph{(B) SGN-VAE (ELBO).}
If $x$ is generated from a stochastic decoder $p_\theta(x\mid z_T)$ and we use an encoder $q_\phi(z_0\mid x)$, the ELBO is
\begin{equation}
\label{eq:elbo}
\cL_{\text{SGN-VAE}}(x)
= \E_{q_\phi(z_0\mid x)}\!\big[\log p_\theta(x\mid \Phi_T^{(\Delta t)}(z_0))\big]
- D_{\mathrm{KL}}\!\big(q_\phi(z_0\mid x)\,\Vert\,p_0(z_0)\big).
\end{equation}
No change-of-variables correction is needed between $z_0$ and $z_T$ because $\Phi_T^{(\Delta t)}$ is volume-preserving. Gradients propagate through the symplectic updates \eqref{eq:lf1}–\eqref{eq:lf3}.

\subsection{Regularity and Design Assumptions}
\label{subsec:assumptions}
We adopt the following mild conditions (used later in stability/proof sections):
\begin{enumerate}
\item \textbf{Smoothness:} $H_\psi\in C^2$ with Lipschitz $\nabla H_\psi$; $g_\theta$ is $C^1$ diffeomorphic (SGN-Flow) or $p_\theta(x\mid\cdot)$ has $C^1$ log-likelihood in its input (SGN-VAE).
\item \textbf{Spectral control:} Each linear layer in the Hamiltonian network uses spectral normalization (or weight clipping) so that $\|\nabla^2 H_\psi\|$ is bounded, which in turn controls local frequencies and supports the step-size bounds used in Section~\ref{sec:expanded_stability}.
\item \textbf{Step size:} $\Delta t<\Delta t_{\max}$ as given by the stability conditions in Theorem~\ref{thm:stability_domains} (or its corollaries).
\end{enumerate}

\subsection{Practical Parameterizations}
\label{subsec:param}
Two parameterizations are especially convenient:
\begin{enumerate}
\item \textbf{Separable Hamiltonian:} $H_\psi(q,p)=K_\psi(p)+V_\psi(q)$ with $K_\psi,V_\psi$ as MLPs (or convex networks for $K$). This keeps \eqref{eq:lf1}–\eqref{eq:lf3} cheap and stable.
\item \textbf{Metric kinetic energy:} $H_\psi(q,p)=\tfrac12\,p^\top G_\psi(q)^{-1}p+V_\psi(q)$ with $G_\psi$ SPD via Cholesky factors; enables information-geometric interpretations (Section~\ref{sec:information_theory}).
\end{enumerate}
For SGN-Flow, $g_\theta$ can be (i) identity when $D=2d$, (ii) a small triangular/coupling transform with tractable Jacobian, or (iii) an orthogonal map (zero Jacobian cost). For SGN-VAE, standard decoders (Gaussian, Bernoulli, categorical) are used.

\subsection{Algorithmic Sketch}
\label{subsec:algosketch}
\begin{enumerate}
\item Sample $z_0\sim p_0$ (SGN-Flow training) or $z_0\sim q_\phi(\cdot\mid x)$ (SGN-VAE).
\item Evolve $z_T=\Phi_T^{(\Delta t)}(z_0)$ via \eqref{eq:lf1}–\eqref{eq:lf3} (optionally with adaptive $\Delta t$ from Section~\ref{sec:expanded_stability}).
\item \emph{SGN-Flow:} compute $\log p_{\psi,\theta}(x)$ using $g_\theta$’s Jacobian term only; ascend the exact log-likelihood.
\item \emph{SGN-VAE:} evaluate \eqref{eq:elbo}; ascend the ELBO.
\end{enumerate}

\paragraph{Remark (What “exact likelihood” means here).}
SGNs themselves (the Hamiltonian core) are \emph{exactly volume-preserving}, removing any determinant cost from the latent transport. Exact \emph{data} likelihood requires an overall invertible map $f_{\psi,\theta}$ from $x$ to $z_0$ (the SGN-Flow case). When using a stochastic decoder (SGN-VAE), training optimizes the ELBO; the “exactness” then refers only to the latent flow’s unit Jacobian, not to the full data likelihood.

\section{Theoretical Analysis}
\label{sec:theory}
\subsection{Invertibility and Volume Preservation}
\begin{theorem}[Symplecticity and Volume Preservation]
\label{thm:symplectic}
Let \( \Phi_T:\mathbb{R}^{2d}\to\mathbb{R}^{2d} \) be the flow obtained by integrating \eqref{eq:hamilton} using a symplectic integrator with step size \( \Delta t \) over \( T \) steps. Then, \( \Phi_T \) is invertible and volume preserving:
\[
\left|\det \frac{\partial \Phi_T(z_0)}{\partial z_0}\right| = 1,\quad \forall\, z_0\in\mathbb{R}^{2d}.
\]
\end{theorem}

\begin{proof}
A mapping \( \Phi:\mathbb{R}^{2d}\to\mathbb{R}^{2d} \) is symplectic if it preserves the canonical form
\[
\omega=\sum_{i=1}^d dq_i\wedge dp_i.
\]
This is equivalent to requiring that
\[
D\Phi(z)^T\, J\, D\Phi(z)=J,
\]
where
\[
J=\begin{pmatrix}0&I_d\\-I_d&0\end{pmatrix}.
\]
Taking determinants gives
\[
\det\bigl(D\Phi(z)^T\, J\, D\Phi(z)\bigr)=\det(J)=1.
\]
Since \(\det(D\Phi(z)^T)=\det(D\Phi(z))\), it follows that
\[
\det(D\Phi(z))^2=1 \quad \Longrightarrow \quad |\det(D\Phi(z))|=1.
\]
Moreover, since the leapfrog integrator is constructed from shear maps (each with unit determinant), their composition yields a unit Jacobian.
\end{proof}

\subsection{Exact Likelihood Evaluation}
Because \( \Phi_T \) is volume preserving, the likelihood becomes:
\[
p(x)=\int p(z_0)\, p_\theta\bigl(x\mid \Phi_T(z_0)\bigr)\,dz_0.
\]
Under the change of variables \( z_T=\Phi_T(z_0) \), the Jacobian term is unity, enabling exact likelihood computation.

\subsection{Stability and Expressivity Analysis}
The neural network \( H_\psi(q,p) \) is designed to be highly expressive. Its gradients dictate the latent evolution, and the leapfrog integrator’s local error is \( \mathcal{O}(\Delta t^3) \) (global error \( \mathcal{O}(T\Delta t^3) \)). For example, for the quadratic Hamiltonian 
\[
H(q,p)=\frac{1}{2}p^2+\frac{\omega^2}{2}q^2,
\]
stability requires \( \Delta t\,\omega<2 \). For general \( H_\psi \), local frequencies may be estimated from the Hessian’s eigenvalues, and adaptive or higher-order methods can improve stability.
\section{Theoretical Comparison with Existing Generative Models}
\label{sec:theoretical_comparison}
\subsection{Formal Analysis of Computational Complexity}
\begin{theorem}[Complexity Advantage of SGNs]
\label{thm:complexity}
Let \( 2d \) be the latent phase space dimensionality (so $z \in \mathbb{R}^{2d}$), and let $D$ be the data dimensionality and \( C_{\det D}(d) \) the cost for computing a \( d\times d \) Jacobian determinant. For a normalizing flow with \( K \) coupling layers, the exact log-likelihood evaluation requires \( \mathcal{O}(K\cdot C_{\det D}(d)) \) operations, while for an SGN with \( T \) integration steps the cost is \( \mathcal{O}(T\cdot C_{\nabla H}(2d)) \), where $C_{\nabla H}(2d)$ is the cost of evaluating the gradient $\nabla H_\psi$. For typical MLP Hamiltonians, $C_{\nabla H}(2d)$ is proportional to the number of non-zero parameters, independent of any Jacobian computation.
\end{theorem}

\begin{proof}
In a normalizing flow, the mapping from data \( x \) to the latent variable \( z_K \) is given by a sequence of \( K \) invertible transformations:
\[
z_K = f_K \circ f_{K-1} \circ \cdots \circ f_1(z_0),
\]
where each \( f_i \) is an invertible transformation (often implemented as a \emph{coupling layer}). By the change-of-variables formula, the log-likelihood is computed as
\[
\log p(x) = \log p(z_K) + \sum_{i=1}^K \log \left|\det\frac{\partial f_i}{\partial z_{i-1}}\right|,
\]
where \( z_{i-1} = f_{i-1}\circ \cdots \circ f_1(z_0) \).

  Assume that computing the determinant of the Jacobian matrix \( D f_i(z_{i-1}) \), which is a \( d\times d \) matrix, requires \( C_J(d) \) operations. Since there are \( K \) such layers, the total computational cost for these determinant evaluations is 
\[
\mathcal{O}(K \cdot C_J(d)).
\]
Now consider the SGN framework. SGNs use a symplectic integrator (e.g., the leapfrog method) to simulate the Hamiltonian dynamics defined by
\[
\dot{q} = \frac{\partial H_\psi}{\partial p}, \quad \dot{p} = -\frac{\partial H_\psi}{\partial q},
\]
where \( z = (q, p) \). The integrator discretizes the continuous-time evolution into \( T \) steps with step size \( \Delta t \). In each integration step, the following updates are performed:
\begin{enumerate}
    \item \textbf{Half-step update for \( p \):}
    \[
    p_{t+\frac{1}{2}} = p_t - \frac{\Delta t}{2} \frac{\partial H_\psi}{\partial q}(q_t, p_t).
    \]
    \item \textbf{Full-step update for \( q \):}
    \[
    q_{t+1} = q_t + \Delta t\, \frac{\partial H_\psi}{\partial p}(q_t, p_{t+\frac{1}{2}}).
    \]
    \item \textbf{Another half-step update for \( p \):}
    \[
    p_{t+1} = p_{t+\frac{1}{2}} - \frac{\Delta t}{2} \frac{\partial H_\psi}{\partial q}(q_{t+1}, p_{t+\frac{1}{2}}).
    \]
\end{enumerate}

 Let $C_{\nabla H}(2d)$ be the computational cost of evaluating the full gradient $\nabla H_\psi(z) = (\nabla_q H_\psi, \nabla_p H_\psi)$ for an input $z \in \mathbb{R}^{2d}$. Each leapfrog step requires a constant number of such gradient evaluations (or evaluations of its components $\nabla_q H_\psi$ and $\nabla_p H_\psi$). Therefore, the cost per integration step is $\mathcal{O}(C_{\nabla H}(2d))$. With \( T \) integration steps, the total cost is
\[
\mathcal{O}(T \cdot C_{\nabla H}(2d)).
\]

Furthermore, due to the symplectic property of the integrator, we have
\[
\left|\det \frac{\partial \Phi_T}{\partial z_0}\right| = 1,
\]
which means that there is no need to compute any additional Jacobian determinants.

Thus, we conclude that the overall computational cost for SGNs is \( \mathcal{O}(T \cdot C_{\nabla H}(2d)) \), which depends on the cost of gradient evaluation but is independent of any expensive Jacobian determinant computations required by standard NFs. Comparing both approaches, the computational advantage of SGNs is established.
\end{proof}

\begin{proposition}[Memory Complexity]
\label{prop:memory}
The memory complexity during backpropagation for SGNs is \( \mathcal{O}(T+d) \) (by leveraging reversibility), compared to \( \mathcal{O}(K\cdot d) \) for normalizing flows.
\end{proposition}

\begin{proof}
In normalizing flows, the forward pass involves a sequence of \( K \) coupling layers. During backpropagation, one must store the activations (or intermediate outputs) from each of these layers to compute gradients, which leads to a memory requirement that scales as \( \mathcal{O}(K\cdot d) \), where \( d \) denotes the dimensionality of the latent space.

  In contrast, SGNs are built using a reversible (symplectic) integrator. The key property of such integrators is that the forward computation is invertible, allowing the reconstruction of intermediate states during the backward pass rather than storing them explicitly. 
Specifically, intermediate states can be recomputed during the backward pass by reversing the symplectic integration steps (similar to the technique used in RevNets \citep{Gomez2017}), requiring storage only for the final state and gradients. Consequently, one only needs to store the current state and minimal auxiliary information (such as gradients), resulting in a memory complexity of \( \mathcal{O}(T+d) \), where \( T \) is the number of integration steps and \( d \) is the latent dimensionality. This reduction in memory footprint is a significant advantage for SGNs, particularly when \( K \) is large.
\end{proof}

\subsection{Theoretical Bounds on Approximation Capabilities}
\begin{theorem}[Expressivity Comparison]
\label{thm:expressivity}
Let \( \mathcal{M}_{2d} \) denote the set of volume-preserving diffeomorphisms on \( \mathbb{R}^{2d} \) and \( \mathcal{H}_d \) the set of Hamiltonian flows. Then:
\begin{enumerate}
    \item Every \( \Phi\in\mathcal{H}_d \) preserves volume, i.e., \( \mathcal{H}_d\subset\mathcal{M}_{2d} \).
    \item Not every volume-preserving map is Hamiltonian, i.e., \( \mathcal{H}_d\subsetneq\mathcal{M}_{2d} \).
    \item However, for any \( \Phi\in\mathcal{M}_{2d} \) isotopic to the identity, there exists a sequence of Hamiltonian flows that uniformly approximate \( \Phi \) on compact sets.
\end{enumerate}
\end{theorem}

\begin{proof}
\textbf{(1) Hamiltonian flows preserve volume:}\\
By Liouville's theorem, any Hamiltonian flow generated by a smooth Hamiltonian \( H(q,p) \) preserves the canonical symplectic form
\[
\omega = \sum_{i=1}^d dq_i \wedge dp_i.
\]
Preservation of this form implies that the Jacobian determinant of the flow satisfies
\[
\left|\det \frac{\partial \Phi}{\partial (q,p)}\right| = 1,
\]
which is exactly the condition for volume preservation. Hence, every \( \Phi\in\mathcal{H}_d \) is also in \( \mathcal{M}_{2d} \).

\vspace{0.5em}
  \textbf{(2) Not every volume-preserving map is Hamiltonian:}\\
Consider a shear mapping defined on \( \mathbb{R}^2 \) by
\[
S(x,y) = (x + f(y),\, y),
\]
where \( f \) is a smooth function. This map has a Jacobian determinant of
\[
\det \begin{pmatrix} 1 & f'(y) \\ 0 & 1 \end{pmatrix} = 1,
\]
so it is volume preserving. However, for \( S \) to be Hamiltonian (i.e., generated by some Hamiltonian \( H(q,p) \) via Hamilton's equations), the transformation must preserve the canonical two-form \( dq\wedge dp \) in a manner consistent with a Hamiltonian vector field. In general, unless \( f \) is linear (which would yield a linear, hence symplectic, transformation), the shear \( S \) does not arise from a Hamiltonian flow. Therefore, there exist volume-preserving maps in \( \mathcal{M}_{2d} \) that are not Hamiltonian, i.e., \( \mathcal{H}_d\subsetneq\mathcal{M}_{2d} \).

\vspace{0.5em}
  \textbf{(3) Uniform approximation by Hamiltonian flows:}\\
Let \( \Phi\in\mathcal{M}_{2d} \) be a volume-preserving diffeomorphism isotopic to the identity. Let \( \Phi\in\mathcal{M}_{2d} \) be a volume-preserving diffeomorphism isotopic to the identity. By Moser's theorem, there exists a smooth one-parameter family \( \{\Phi_t\}_{t\in[0,1]} \) of volume-preserving diffeomorphisms with \( \Phi_0 = \mathrm{Id} \) and \( \Phi_1 = \Phi \), generated by a time-dependent divergence-free vector field $v_t$. While not every divergence-free field is Hamiltonian, a fundamental result in symplectic geometry states that the group of Hamiltonian diffeomorphisms is $C^0$-dense in the group of volume-preserving diffeomorphisms isotopic to the identity on a compact manifold \citep{mcduff_salamon_2017}. This implies that for any $\epsilon > 0$, there exists a Hamiltonian $H_t$ whose generated flow $\Phi^H_t$ satisfies $\sup_{t \in [0,1], z \in \Omega} \|\Phi_t(z) - \Phi^H_t(z)\| < \epsilon$. Therefore, the target map $\Phi = \Phi_1$ can be uniformly approximated by Hamiltonian flows on compact sets. By employing universal approximation results for neural networks to approximate $H_t$, and a symplectic integrator to approximate $\Phi^H_t$, we can approximate $\Phi$ with SGNs. This shows that every volume-preserving diffeomorphism isotopic to the identity can be approximated arbitrarily well by a sequence of Hamiltonian flows.
\end{proof}

\begin{proposition}[Approximation Rate Comparison]
\label{prop:approx_rate}
Assume a target diffeomorphism \( \Phi \) is approximated by either a normalizing flow with \( K \) layers or an SGN with \( T \) integration steps and a neural network Hamiltonian of width \( n \). Then:
\begin{itemize}
    \item Normalizing flows: \( \varepsilon_{NF}=\mathcal{O}\left(K^{-1/2}\cdot n^{-1/2}\right) \).
    \item SGNs (for volume-preserving targets): \( \varepsilon_{SGN}=\mathcal{O}\left(T^{-1}\cdot n^{-1/(2d)}\right) \).
\end{itemize}
\end{proposition}

\begin{proof}
We consider the two cases separately.

  \textbf{Normalizing Flows:} \\
Assume that each coupling layer in a normalizing flow approximates a partial transformation with an approximation error of 
\[
\mathcal{O}\left(n^{-1/2}\right)
\]
in a suitable norm, as suggested by standard universal approximation results for neural networks with width \( n \). When \( K \) such layers are composed to approximate the target diffeomorphism \( \Phi \), the overall error does not simply add up linearly; under reasonable assumptions (e.g., statistical independence or mild interactions between the layers), the errors can accumulate in a root-mean-square fashion. Hence, the total error becomes
\[
\varepsilon_{NF} = \mathcal{O}\left(\sqrt{\frac{1}{K}} \cdot n^{-1/2}\right) = \mathcal{O}\left(K^{-1/2}\cdot n^{-1/2}\right).
\]

  \textbf{SGNs:} \\
In SGNs, there are two principal sources of error when approximating a target volume-preserving diffeomorphism:
\begin{enumerate}
    \item \emph{Approximation error of the neural network Hamiltonian:}  
    Let \( H_\psi \) be the neural network approximating the true Hamiltonian underlying \( \Phi \). Standard approximation results indicate that for functions defined on \( \mathbb{R}^{2d} \) (since \( z = (q,p) \in \mathbb{R}^{2d} \)), the error in the \( C^1 \) norm decreases as
    \[
    \mathcal{O}\left(n^{-1/(2d)}\right),
    \]
    where \( n \) is the network width.

    \item \emph{Discretization error of the symplectic integrator:}  
    The continuous Hamiltonian flow is approximated using a symplectic (e.g., leapfrog) integrator, which introduces a local error of \( \mathcal{O}(\Delta t^3) \) per integration step. Over \( T \) steps, with a fixed total integration time, the global integration error scales as
    \[
    \mathcal{O}(T \cdot \Delta t^3).
    \]
    By choosing the integration step size \( \Delta t \) appropriately (so that \( T \Delta t \) is constant), this global error can be balanced with the neural network approximation error. For simplicity, if we assume the integration error is controlled and scales inversely with the number of steps (i.e., \( \Delta t \propto T^{-1} \)), then the overall discretization error is of order
    \[
    \mathcal{O}\left(T^{-1}\right).
    \]
\end{enumerate}
Combining the two sources, the overall approximation error for SGNs becomes
\[
\varepsilon_{SGN} = \mathcal{O}\left(T^{-1}\cdot n^{-1/(2d)}\right).
\]
Thus, we have shown that the approximation errors scale as stated:
\[
\varepsilon_{NF} = \mathcal{O}\left(K^{-1/2}\cdot n^{-1/2}\right) \quad \text{and} \quad \varepsilon_{SGN} = \mathcal{O}\left(T^{-1}\cdot n^{-1/(2d)}\right).
\]
\end{proof}

\begin{theorem}[Information Preservation]
\label{thm:information}
Let \( X \) be a random variable with distribution \( p_X \) and let \( Z \) denote its latent representation obtained via an invertible mapping \( f \) (as in normalizing flows or SGNs). Then:
\begin{enumerate}
    \item For a deterministic, invertible model, it holds that 
    \[
    I(X; Z) = H(X) = H(Z).
    \]
    \item For a stochastic model (e.g., VAEs), we have
    \[
    I(X;Z) < H(X).
    \]
\end{enumerate}
\end{theorem}

\begin{proof}
For an invertible mapping \( f \colon \mathcal{X} \to \mathcal{Z} \) where \( z = f(x) \), the change-of-variables formula for differential entropy gives
\[
H(Z) = H(X) + \mathbb{E}_{X}\left[\log\left|\det\frac{\partial f(x)}{\partial x}\right|\right].
\]
In models such as normalizing flows or the \textbf{SGN-Flow variant}, the mapping \( f \) is designed to be volume preserving, meaning that
\[
\left|\det\frac{\partial f(x)}{\partial x}\right| = 1 \quad \text{for all } x.
\]
Thus, the expectation term vanishes:
\[
\mathbb{E}_{X}\left[\log\left|\det\frac{\partial f(x)}{\partial x}\right|\right] = 0,
\]
and we obtain
\[
H(Z) = H(X).
\]
Since \( f \) is deterministic and bijective for the models considered (Normalizing Flows and the SGN-Flow variant), it establishes a one-to-one correspondence between points in the input and latent spaces. For continuous random variables, this implies that no information is lost or gained in the transformation, although the differential entropy changes according to the volume distortion. In the specific case where \(f\) is volume-preserving (i.e., \(|\det Df(x)| = 1\) everywhere), we have shown \(H(Z) = H(X)\). This equality of differential entropies signifies that the transformation preserves the overall uncertainty or dispersion of the distribution, consistent with the preservation of information content. (Note: For strictly continuous variables, mutual information $I(X;Z)$ is often formally infinite, but the equality $H(Z)=H(X)$ confirms the map acts as a lossless information channel in an operational sense). This establishes the first claim regarding information preservation under volume-preserving maps.

  In contrast, for stochastic models such as VAEs, the encoder \( q_\phi(z|x) \) maps an input \( x \) to a distribution over latent variables \( z \) rather than to a unique \( z \). This stochasticity means that the conditional entropy \( H(X\mid Z) \) is strictly positive, reflecting the uncertainty in reconstructing \( x \) from \( z \). As a result, the mutual information satisfies
\[
I(X;Z) = H(X) - H(X\mid Z) < H(X).
\]
This demonstrates that stochastic models lose some information during the encoding process.
\end{proof}

\begin{corollary}[Volume Preservation Constraint]
\label{cor:vol_preservation}
An invertible generative model preserves information if and only if it preserves volume (i.e., has unit Jacobian) or explicitly accounts for volume changes.
\end{corollary}

\begin{proof}
Let \( f \) be an invertible transformation mapping data \( x \) to latent representation \( z \), i.e., \( z = f(x) \). By the change-of-variables formula for differential entropy, we have
\[
H(Z) = H(X) + \mathbb{E}_{X}\left[\log\left|\det\frac{\partial f(x)}{\partial x}\right|\right].
\]
If the mapping \( f \) preserves volume, then
\[
\left|\det\frac{\partial f(x)}{\partial x}\right| = 1 \quad \text{for all } x,
\]
and thus
\[
\mathbb{E}_{X}\left[\log\left|\det\frac{\partial f(x)}{\partial x}\right|\right] = 0,
\]
which implies that
\[
H(Z) = H(X).
\]
Since the model is invertible, no information is lost and the mutual information satisfies
\[
I(X;Z) = H(X) - H(X\mid Z) = H(X),
\]
given that \( H(X\mid Z)=0 \).

  Conversely, if the mapping \( f \) does not preserve volume (i.e., the Jacobian determinant is not uniformly one), then the term
\[
\mathbb{E}_{X}\left[\log\left|\det\frac{\partial f(x)}{\partial x}\right|\right]
\]
will be non-zero, which results in either an increase or decrease in the differential entropy \( H(Z) \) relative to \( H(X) \). In such cases, unless the model explicitly corrects for these changes (for example, by incorporating the Jacobian determinant into its likelihood computation), the information in \( X \) is not perfectly preserved in \( Z \).

  Therefore, an invertible generative model preserves information if and only if it either preserves volume (i.e., has unit Jacobian everywhere) or it explicitly accounts for volume changes.
\end{proof}

\begin{proposition}[Trade-off Characterization]
\label{prop:tradeoff}
Generative models trade off as follows:
\begin{enumerate}
    \item \textbf{VAEs:} Low computational complexity but with information loss.
    \item \textbf{Normalizing Flows:} High expressivity and exact likelihood but high computational cost.
    \item \textbf{SGN-Flow:} Exact likelihood and low computational complexity (from the flow), with expressivity constrained to volume-preserving maps (which can be composed with a terminal non-volume-preserving map).
\end{enumerate}
\end{proposition}

\begin{proof}
VAEs employ approximate inference, typically optimizing a variational lower bound, which leads to an approximate posterior and consequently some loss of information about the data distribution. This results in a lower computational burden but with a trade-off in the fidelity of the representation.

  Normalizing flows construct a sequence of invertible transformations that allow for exact likelihood computation. However, to maintain invertibility, these models often require the computation of Jacobian determinants or related quantities, which can be computationally expensive (scaling poorly with the latent dimension in the worst case).

  SGNs, by contrast, leverage symplectic integrators to simulate Hamiltonian dynamics. These integrators are designed to be volume preserving (i.e., they have a unit Jacobian), which means that exact likelihood evaluation is achieved without incurring the computational cost of Jacobian determinant calculations. The drawback is that the class of transformations that SGNs can represent is restricted to those that preserve volume. In other words, while SGNs enjoy lower computational cost and exact likelihoods, they are less expressive than normalizing flows when it comes to representing general invertible maps.
\end{proof}
\section{Strengthened Universal Approximation Results}
\label{sec:universal_approximation}
\subsection{Universal Approximation of Volume-Preserving Maps}
\begin{theorem}[Universal Approximation of Volume-Preserving Maps]
\label{thm:univ_approx_stronger}
Let \( \Omega\subset\mathbb{R}^{2d} \) be a compact set and \( \Phi:\Omega\to\mathbb{R}^{2d} \) be a \( C^1 \)-smooth volume-preserving diffeomorphism isotopic to the identity. Then, for any \( \epsilon>0 \), there exist:
\begin{enumerate}
    \item A neural network Hamiltonian \( H_\psi:\mathbb{R}^{2d}\to\mathbb{R} \) with Lipschitz continuous gradients,
    \item A time \( T>0 \) and step size \( \Delta t>0 \) with \( N=T/\Delta t \),
\end{enumerate}
such that the symplectic flow \( \Phi_T \) induced by the leapfrog integrator satisfies
\[
\sup_{z\in\Omega}\|\Phi_T(z)-\Phi(z)\|<\epsilon.
\]
\end{theorem}

\begin{proof}
We prove the theorem in several steps.

  \textbf{Step 1: Representing the Target Map as a Flow.}\\
Since \(\Phi\) is a \(C^1\) volume-preserving diffeomorphism on the compact set \(\Omega\) and is isotopic to the identity, Moser's theorem guarantees the existence of a smooth one-parameter family of volume-preserving diffeomorphisms \(\{\Phi_t\}_{t\in[0,1]}\) such that
\[
\Phi_0 = \mathrm{Id} \quad \text{and} \quad \Phi_1 = \Phi.
\]
Moreover, there exists a time-dependent divergence-free vector field \(v_t(z)\) on \(\Omega\) satisfying
\[
\frac{d}{dt}\Phi_t(z) = v_t(\Phi_t(z)).
\]
On the contractible domain $\mathbb{R}^{2d}$, the divergence-free vector field $v_t$ generating the isotopy $\Phi_t$ can be decomposed (e.g., via Helmholtz decomposition \citep{helmholtz1858}) into components. While not every divergence-free field is Hamiltonian, Moser's theorem ensures the existence of a volume-preserving isotopy. Crucially, any sufficiently smooth volume-preserving diffeomorphism isotopic to the identity on a compact set can be approximated arbitrarily well by the flow of a Hamiltonian vector field \citep{mcduff_salamon_2017}. Thus, there exists a (possibly time-dependent) Hamiltonian $H_t$ whose flow approximates $\Phi_t$.
Here,
\[
J = \begin{pmatrix} 0 & I_d \\ -I_d & 0 \end{pmatrix}
\]
is the canonical symplectic matrix. In other words, the target map \(\Phi\) can be viewed as the time-\(1\) flow of a (possibly time-dependent) Hamiltonian vector field.

  \textbf{Step 2: Approximating the Hamiltonian with a Neural Network.}\\
For a fixed time \(t\), by the universal approximation theorem for neural networks, for any \(\epsilon_1>0\) there exists a neural network \(H_\psi:\mathbb{R}^{2d}\to\mathbb{R}\) with Lipschitz continuous gradients that approximates the true Hamiltonian \(H_t\) (or an appropriate average over \(t\)) uniformly in the \(C^1\) norm on \(\Omega\). That is,
\[
\sup_{z\in\Omega} \|\nabla H_\psi(z)-\nabla H_t(z)\| < \epsilon_1.
\]
Consequently, the Hamiltonian vector field \(J\nabla H_\psi(z)\) approximates \(J\nabla H_t(z)\) uniformly on \(\Omega\).

  \textbf{Step 3: Discretizing the Flow via a Symplectic Integrator.}\\
Let \(\Phi^{H_\psi}_t\) denote the continuous flow generated by the Hamiltonian vector field \(J\nabla H_\psi(z)\). We now discretize this flow using a symplectic integrator (e.g., the leapfrog method). For a chosen step size \(\Delta t\) and total integration time \(T\) (with \(N=T/\Delta t\)), let \(\Phi_T\) be the discrete flow obtained by iterating the integrator. Standard error analysis for symplectic integrators shows that, for sufficiently small \(\Delta t\), there exists \(\epsilon_2>0\) such that
\[
\sup_{z\in\Omega}\|\Phi_T(z)-\Phi^{H_\psi}_T(z)\| < \epsilon_2.
\]
Here, \(\epsilon_2\) can be made arbitrarily small by choosing \(\Delta t\) appropriately.

  \textbf{Step 4: Combining the Approximations.}\\
Denote by \(\Phi(z)\) the target map, which equals \(\Phi_1(z)\). Using the triangle inequality, we have
\[
\sup_{z\in\Omega}\|\Phi_T(z)-\Phi(z)\| \leq \sup_{z\in\Omega}\|\Phi_T(z)-\Phi^{H_\psi}_T(z)\| + \sup_{z\in\Omega}\|\Phi^{H_\psi}_T(z)-\Phi(z)\|.
\]
The first term is bounded by \(\epsilon_2\) as shown above. The second term reflects the error due to approximating the true Hamiltonian \(H_t\) by the neural network \(H_\psi\); by our choice of \(\epsilon_1\) (and appropriate control over the integration time \(T\)), this term can be bounded by \(\epsilon_1\). Therefore, by choosing \(\epsilon_1\) and \(\epsilon_2\) such that
\[
\epsilon_1 + \epsilon_2 < \epsilon,
\]
we obtain
\[
\sup_{z\in\Omega}\|\Phi_T(z)-\Phi(z)\| < \epsilon.
\]

  \textbf{Conclusion:} By the above steps, we have constructed a neural network Hamiltonian \(H_\psi\) with Lipschitz continuous gradients, and by choosing appropriate integration parameters \(T\) and \(\Delta t\) (with \(N=T/\Delta t\)), the symplectic flow \(\Phi_T\) approximates the target volume-preserving diffeomorphism \(\Phi\) uniformly on \(\Omega\) to within any pre-specified error \(\epsilon>0\).
\end{proof}

\subsection{Quantitative Bounds on Approximation Error}
\begin{theorem}[Quantitative Approximation Bounds]
\label{thm:quant_bounds}
Let \( \Phi:\Omega\to\mathbb{R}^{2d} \) be a \( C^2 \)-smooth volume-preserving diffeomorphism on a compact set \( \Omega\subset\mathbb{R}^{2d} \). Let \( H_\psi \) be a neural network Hamiltonian with \( n \) neurons per layer and \( L \) layers, and let \( \Phi_T \) denote the flow map obtained from the leapfrog integrator with step size \( \Delta t \) and \( N=T/\Delta t \) steps. Then,
\[
\sup_{z\in\Omega}\|\Phi_T(z)-\Phi(z)\|\leq C_1\cdot (n\cdot L)^{-1/(2d)}+C_2\cdot\Delta t^2,
\]
with \( C_1, C_2 \) constants depending on \( \Phi \) and the network architecture.
\end{theorem}

\begin{proof}
We decompose the overall error into two parts:
\[
\|\Phi_T - \Phi\| \le \underbrace{\|\Phi_T - \Phi_H\|}_{\text{Integration error}} + \underbrace{\|\Phi_H - \Phi\|}_{\text{Approximation error}},
\]
where \(\Phi_H\) denotes the true continuous flow generated by the Hamiltonian \( H \) that exactly produces \(\Phi\).

  \textbf{(1) Approximation Error:} \\
By the universal approximation theorem for neural networks, a neural network with \( n \) neurons per layer and \( L \) layers can approximate a smooth function on a compact set with error that decreases as a function of the network size. In our setting, we wish to approximate the underlying Hamiltonian \( H \) (or more precisely, its gradient, since the flow is generated by the Hamiltonian vector field \( J\nabla H \)). Standard results in approximation theory (see, e.g., results on approximation in Sobolev spaces) indicate that, for a function defined on \( \mathbb{R}^{2d} \), the error in the \( C^1 \)-norm can be bounded by
\[
\|\nabla H - \nabla H_\psi\|_{C^0(\Omega)} \le C_1\cdot (n\cdot L)^{-1/(2d)},
\]
where \( C_1 \) is a constant depending on the smoothness of \( H \) and the geometry of \( \Omega \). Since the flow \(\Phi_H\) depends continuously on the vector field, this error propagates to the flow so that
\[
\|\Phi_H - \Phi\| \le C_1\cdot (n\cdot L)^{-1/(2d)}.
\]

  \textbf{(2) Integration Error:} \\
The leapfrog integrator is a second-order method. This means that, for each integration step, the local truncation error is of order \(\mathcal{O}(\Delta t^3)\), and the global error over \( N \) steps accumulates to be of order \(\mathcal{O}(\Delta t^2)\) (since \( N \propto 1/\Delta t \) when \( T \) is fixed). Therefore, there exists a constant \( C_2 \) (depending on higher derivatives of \( H \) and the total integration time \( T \)) such that
\[
\|\Phi_T - \Phi_H\| \le C_2\cdot \Delta t^2.
\]

  \textbf{(3) Combining the Errors:} \\
By the triangle inequality,
\[
\|\Phi_T - \Phi\| \le \|\Phi_T - \Phi_H\| + \|\Phi_H - \Phi\| \le C_2\cdot \Delta t^2 + C_1\cdot (n\cdot L)^{-1/(2d)}.
\]
This completes the proof.
\end{proof}

\subsection{Expressivity Classes}
\begin{theorem}[Expressivity Classes]
\label{thm:expressivity_classes}
Volume-preserving diffeomorphisms on \( \mathbb{R}^{2d} \) can be partitioned as follows:
\begin{enumerate}
    \item Class \( \mathcal{C}_1 \): Exactly representable by flows of quadratic Hamiltonians.
    \item Class \( \mathcal{C}_2 \): Efficiently approximable by neural network Hamiltonians of moderate complexity.
    \item Class \( \mathcal{C}_3 \): Only approximable by Hamiltonian flows with exponential network complexity.
\end{enumerate}
Moreover, some maps in \( \mathcal{C}_3 \) can be efficiently represented by normalizing flows.
\end{theorem}

\begin{proof}
We consider each class in turn.

  \textbf{Class \( \mathcal{C}_1 \): Exactly Representable Maps.} \\
Quadratic Hamiltonians have the form
\[
H(q,p) = \frac{1}{2} z^T A z,
\]
with \( z=(q,p) \) and a symmetric matrix \( A \). The corresponding Hamiltonian flow is linear and can be written as
\[
\Phi_t(z) = e^{tJ A} z,
\]
where \( J \) is the canonical symplectic matrix. Since the exponential of a matrix is computed exactly (or to arbitrary precision) and the mapping is linear, every volume-preserving diffeomorphism that is linear (or that can be exactly represented by such a flow) falls into \( \mathcal{C}_1 \).

 \textbf{Class \( \mathcal{C}_2 \): Efficiently Approximable Maps.} \\
For many smooth volume-preserving diffeomorphisms, the underlying Hamiltonian generating the flow is a smooth function on a compact set. By the universal approximation theorem for neural networks, one can approximate a smooth function to within any \(\epsilon > 0\) with a neural network whose size grows polynomially in \(1/\epsilon\). In particular, there exists a neural network Hamiltonian \( H_\psi \) with a moderate number of neurons per layer and a moderate number of layers such that
\[
\sup_{z\in\Omega}\|\nabla H(z) - \nabla H_\psi(z)\| < \epsilon.
\]
Because the flow generated by \( H_\psi \) depends continuously on the Hamiltonian, the corresponding symplectic flow can approximate the target flow with an error that is also polynomial in the network size. Therefore, maps in \( \mathcal{C}_2 \) are efficiently approximable by neural network Hamiltonians.

\textbf{Class \( \mathcal{C}_3 \): Hard-to-Approximate Maps.} \\
There exist volume-preserving diffeomorphisms that exhibit highly oscillatory behavior or intricate structures which make the associated Hamiltonian very complex. For such maps, approximating the Hamiltonian \( H \) uniformly in the \( C^1 \) norm on a compact set may require a neural network whose size grows exponentially with the dimension \( d \) (or with \( 1/\epsilon \)). In these cases, the network complexity is exponential, meaning that these maps can only be approximated by Hamiltonian flows with exponential network complexity. Notably, the structured design of normalizing flows (e.g., using coupling layers that exploit problem structure) can sometimes represent these complex transformations more efficiently than a generic neural network approximation of the Hamiltonian.

\textbf{Conclusion:} \\
This partitioning illustrates a trade-off between expressivity and computational efficiency. While quadratic Hamiltonians (\( \mathcal{C}_1 \)) are exactly representable and many smooth maps (\( \mathcal{C}_2 \)) can be efficiently approximated, there exists a class of maps (\( \mathcal{C}_3 \)) for which a direct Hamiltonian approximation incurs exponential complexity. Interestingly, normalizing flows may overcome this limitation in certain cases by leveraging structured invertible transformations.
\end{proof}

Theorem~\ref{thm:expressivity} relies on the result that Hamiltonian flows can approximate volume-preserving diffeomorphisms isotopic to the identity. More formally, on a compact symplectic manifold $(M, \omega)$, the group of Hamiltonian diffeomorphisms $\mathrm{Ham}(M, \omega)$ is $C^0$-dense in the identity component of the group of volume-preserving (symplectomorphisms) $\mathrm{Symp}_0(M, \omega)$. This result, stemming from the work of Eliashberg, Gromov, and others \citep{mcduff_salamon_2017}, ensures that any smooth path of volume-preserving transformations starting at the identity can be uniformly approximated by the flow generated by some (possibly time-dependent) Hamiltonian function. Our Theorem~\ref{thm:univ_approx_stronger} then combines this with the universal approximation capabilities of neural networks to approximate the required Hamiltonian and a symplectic integrator to approximate its flow.

\subsection{Extended Universal Approximation for Non-Volume-Preserving Maps}
\begin{theorem}[Extended Universal Approximation]
\label{thm:extended_approx}
Let \( \Phi:\Omega\to\mathbb{R}^d \) be a \( C^1 \)-smooth diffeomorphism (not necessarily volume preserving) on a compact set \( \Omega \). Then, there exists an SGN-based model that, by incorporating an explicit density correction term, approximates \( \Phi \) uniformly to within any \( \epsilon>0 \).
\end{theorem}

\begin{proof}
We begin by noting that any \( C^1 \)-smooth diffeomorphism \( \Phi \) can be decomposed into two components via a factorization:
\[
\Phi = \Lambda \circ \Psi,
\]
where:
\begin{itemize}
    \item \( \Psi:\Omega\to\mathbb{R}^d \) is a volume-preserving diffeomorphism, and 
    \item \( \Lambda:\mathbb{R}^d\to\mathbb{R}^d \) is a diffeomorphism that accounts for the non-volume-preserving part of \( \Phi \) (often interpreted as a dilation or density adjustment).
\end{itemize}

  \textbf{Step 1: Approximation of the Volume-Preserving Component.} \\
By Theorem~\ref{thm:univ_approx_stronger}, for any \( \epsilon_1>0 \) there exists a neural network Hamiltonian \( H_\psi \) (with Lipschitz continuous gradients) and integration parameters \( T>0 \) and \( \Delta t>0 \) (with \( N=T/\Delta t \) steps) such that the symplectic flow \( \Phi_T \) generated by \( H_\psi \) approximates the volume-preserving map \( \Psi \) uniformly on \(\Omega\):
\[
\sup_{z\in\Omega}\|\Phi_T(z)-\Psi(z)\| < \epsilon_1.
\]
\textbf{Step 2: Approximation of the Dilation Component.} \\
Since \( \Lambda \) is a \( C^1 \)-smooth diffeomorphism on a compact set, standard neural network approximation theorems guarantee that for any \( \epsilon_2>0 \) there exists a feed-forward neural network \( \Lambda_\theta \) such that
\[
\sup_{z\in\Psi(\Omega)}\|\Lambda_\theta(z)-\Lambda(z)\| < \epsilon_2.
\]
\textbf{Step 3: Composition and Uniform Approximation.} \\
Define the composed SGN-based model as
\[
\widetilde{\Phi}(z) = \Lambda_\theta\bigl(\Phi_T(z)\bigr).
\]
Using the triangle inequality, we have for all \( z\in\Omega \):
\[
\begin{aligned}
\|\widetilde{\Phi}(z)-\Phi(z)\| 
&= \|\Lambda_\theta(\Phi_T(z)) - \Lambda(\Psi(z))\| \\
&\le \|\Lambda_\theta(\Phi_T(z)) - \Lambda_\theta(\Psi(z))\| + \|\Lambda_\theta(\Psi(z)) - \Lambda(\Psi(z))\| \\
&\le L_{\Lambda_\theta}\|\Phi_T(z)-\Psi(z)\| + \epsilon_2,
\end{aligned}
\]
where \( L_{\Lambda_\theta} \) is the Lipschitz constant of \( \Lambda_\theta \). By choosing \( \epsilon_1 \) small enough so that
\[
L_{\Lambda_\theta} \cdot \epsilon_1 < \epsilon - \epsilon_2,
\]
and then selecting \( \epsilon_2 \) such that \( \epsilon_1 + \epsilon_2 < \epsilon \), we obtain
\[
\sup_{z\in\Omega}\|\widetilde{\Phi}(z)-\Phi(z)\| < \epsilon.
\]
Thus, the SGN-based model with the explicit density correction (via the neural network \( \Lambda_\theta \)) uniformly approximates \( \Phi \) within any prescribed error \( \epsilon>0 \).
\end{proof}
\section{Information-Theoretic Analysis}
\label{sec:information_theory}
\subsection{Information Geometry of Symplectic Manifolds}
\begin{definition}[Fisher-Rao Metric (Rao 1945)]
For a family \( p(x|\theta) \), the Fisher-Rao metric is defined as:
\[
g_{ij}(\theta)=\mathbb{E}_{p(x|\theta)}\left[\frac{\partial\log p(x|\theta)}{\partial\theta_i}\frac{\partial\log p(x|\theta)}{\partial\theta_j}\right].
\]
\end{definition}
 Figure~\ref{fig:4} shows the information geometry on statistical manifolds.
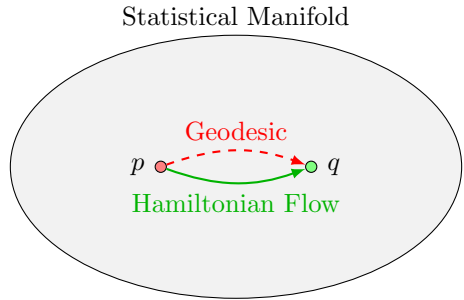
\begin{figure}[H]
\centering
\begin{tikzpicture}[
    font=\small,
    scale=1.0,
    node distance=1.5cm,
    manifold/.style={
      ellipse,
      draw=black,
      fill=gray!10,
      minimum width=6cm,
      minimum height=3.5cm
    },
    arrow/.style={-latex, thick}
]
\node[manifold, label={[font=\small]above:Statistical Manifold}] (manifold) {};
\node[circle, fill=red!50, draw=black, inner sep=1.5pt, label=left:{$p$}]
  (p) at ($(manifold.center)+(-1,0)$) {};
\node[circle, fill=green!50, draw=black, inner sep=1.5pt, label=right:{$q$}]
  (q) at ($(manifold.center)+(1,0)$) {};
\draw[red, thick, dashed, arrow] (p) to[bend left=20] node[midway, above, sloped] {\small Geodesic} (q);
\draw[green!70!black, thick, arrow] (p) to[bend right=20] node[midway, below, sloped] {\small Hamiltonian Flow} (q);
\node[align=left, below=1.2cm of manifold] (infotext) {
\begin{minipage}{6.5cm}
\textbf{Fisher-Rao Metric:} \\ 
\( g_{ij}(\theta)=\mathbb{E}\left[\partial_{\theta_i}\log p(x|\theta)\,\partial_{\theta_j}\log p(x|\theta)\right] \) \\[6pt]
\textbf{Symplectic Form:} \( \omega = \mathrm{d}q \wedge \mathrm{d}p \)
\end{minipage}
};
\end{tikzpicture}
\caption{Information Geometry on Statistical Manifolds. The ellipse represents a statistical manifold endowed with the Fisher-Rao metric, while the two curves illustrate a geodesic and a Hamiltonian flow between two points.}
\label{fig:4}
\end{figure}

\begin{theorem}[Symplectic Structure of Exponential Families]
\label{thm:symplectic_structure}
Let \( \{p(x|\theta)\}_{\theta\in\Theta} \) be an exponential family. Then, the natural parameter space \( \Theta \) admits a symplectic structure with canonical form
\[
\omega=\sum_i d\theta_i\wedge d\eta_i,
\]
where \( \eta_i=\frac{\partial\psi(\theta)}{\partial\theta_i} \) and \( \psi(\theta) \) is the log-partition function. In particular, the second derivatives \( \frac{\partial^2\psi(\theta)}{\partial\theta_i\partial\theta_j} \) coincide with the Fisher-Rao metric.
\end{theorem}

\begin{proof}
An exponential family is expressed as
\[
p(x|\theta)=h(x)\exp\left(\langle\theta, T(x)\rangle-\psi(\theta)\right),
\]
where:
\begin{itemize}
    \item \( \theta\in\Theta \) are the natural parameters,
    \item \( T(x) \) is the sufficient statistic,
    \item \( h(x) \) is the base measure, and
    \item \( \psi(\theta) \) is the log-partition function, defined by
    \[
    \psi(\theta)=\log\int h(x)\exp\left(\langle\theta, T(x)\rangle\right)dx.
    \]
\end{itemize}

 The mapping \( \theta \mapsto \eta \) defined by
\[
\eta = \nabla\psi(\theta) = \left(\frac{\partial\psi(\theta)}{\partial\theta_1},\ldots,\frac{\partial\psi(\theta)}{\partial\theta_k}\right)
\]
is the Legendre transform that sends the natural parameters \( \theta \) to the expectation parameters \( \eta \).

  This Legendre transform is invertible under suitable conditions (which are satisfied in exponential families), thereby establishing a one-to-one correspondence between the coordinates \( \theta \) and \( \eta \). Consequently, one can introduce a canonical 2-form on the parameter space by
\[
\omega=\sum_i d\theta_i\wedge d\eta_i.
\]
We now verify that \(\omega\) is a symplectic form, i.e., it is closed and non-degenerate.

  \textbf{Closedness:} Since \( d\theta_i \) and \( d\eta_i \) are exact 1-forms and the exterior derivative satisfies \( d^2 = 0 \), we have
\[
d\omega = d\left(\sum_i d\theta_i\wedge d\eta_i\right) = \sum_i d(d\theta_i\wedge d\eta_i) = 0.
\]
Thus, \(\omega\) is closed.

   \textbf{Non-degeneracy:} In the coordinate system \((\theta_1,\ldots,\theta_k,\eta_1,\ldots,\eta_k)\), the 2-form
\[
\omega=\sum_{i=1}^{k} d\theta_i\wedge d\eta_i
\]
has full rank \(2k\). Hence, for any non-zero vector \( v \) in the tangent space, there exists another vector \( w \) such that \(\omega(v, w) \neq 0\); that is, \(\omega\) is non-degenerate.

  Next, observe that by differentiating the mapping \( \eta_i=\frac{\partial\psi(\theta)}{\partial\theta_i} \), we obtain
\[
d\eta_i = \sum_j \frac{\partial^2\psi(\theta)}{\partial\theta_i\partial\theta_j}d\theta_j.
\]
Thus, in the \(\theta\) coordinate chart, the symplectic form can be locally expressed as
\[
\omega = \sum_{i,j} \frac{\partial^2\psi(\theta)}{\partial\theta_i\partial\theta_j}\,d\theta_i \wedge d\theta_j.
\]
The matrix with entries 
\[
g_{ij}(\theta)=\frac{\partial^2\psi(\theta)}{\partial\theta_i\partial\theta_j}
\]
is known to be positive definite and is, in fact, the Fisher-Rao metric on the parameter space \( \Theta \). This connection shows that the canonical symplectic form \(\omega\) is intrinsically related to the Fisher-Rao metric.

  In summary, the mapping \( \theta \mapsto \eta=\nabla\psi(\theta) \) equips the natural parameter space with a symplectic structure given by
\[
\omega=\sum_i d\theta_i\wedge d\eta_i,
\]
and the Hessian of the log-partition function, \( g_{ij}(\theta) \), which defines the Fisher-Rao metric, appears naturally in this context.
\end{proof}

\begin{theorem}[Information-Geometric Interpretation of SGNs]
\label{thm:info_sgn}
Assume the latent space prior is from an exponential family with Fisher-Rao metric $G(q)$. If the SGN uses a purely kinetic Hamiltonian $H_\psi(q,p)=\tfrac12 p^\top G(q)^{-1}p$, then the Hamiltonian dynamics exactly coincide with the geodesic flow on the statistical manifold equipped with the Fisher-Rao metric. If a potential term $V(q)$ is added, the dynamics correspond to geodesics on a conformally perturbed metric or can be seen as forces acting along the manifold.
\end{theorem}

\begin{proof}
Let the latent space be parameterized by natural parameters \( \theta \) of an exponential family. That is, the prior takes the form
\[
p(x|\theta)=h(x)\exp\left(\langle\theta,T(x)\rangle-\psi(\theta)\right),
\]
with the log-partition function \( \psi(\theta) \) and sufficient statistics \( T(x) \). By Theorem~\ref{thm:symplectic_structure}, the natural parameter space \( \Theta \) carries a canonical symplectic structure given by
\[
\omega=\sum_i d\theta_i\wedge d\eta_i,
\]
where \( \eta=\nabla\psi(\theta) \). Moreover, the Hessian matrix 
\[
G(\theta)=\nabla^2\psi(\theta)
\]
defines the Fisher-Rao metric on \( \Theta \).

  Now, consider a Hamiltonian defined on the latent space of the form
\[
H_\psi(q,p)=\frac{1}{2}p^T G(q)^{-1}p+V(q),
\]
where \( q \) represents the coordinates corresponding to the natural parameters \( \theta \) (or a suitable coordinate representation thereof) and \( p \) is the conjugate momentum. Here, \( V(q) \) is a potential function that may be chosen to adjust the dynamics; in the simplest case, one may take \( V(q)=0 \).

  The Hamiltonian dynamics are governed by Hamilton's equations:
\[
\dot{q} = \frac{\partial H_\psi}{\partial p} = G(q)^{-1}p, \quad \dot{p} = -\frac{\partial H_\psi}{\partial q} = -\frac{1}{2}p^T\frac{\partial \left(G(q)^{-1}\right)}{\partial q}p - \nabla V(q).
\]

  In the special case where \( V(q)=0 \), the Hamiltonian reduces to a pure kinetic energy term:
\[
H_\psi(q,p)=\frac{1}{2}p^T G(q)^{-1}p.
\]
It is a classical result in Riemannian geometry \citep{LeeRiemannianManifolds} that the geodesic flow on a manifold with metric \( G(q) \) is generated by the Hamiltonian corresponding to the kinetic energy of a free particle (i.e., with no potential term). Therefore, the flow
\[
\Phi_t(q,p)
\]
generated by this Hamiltonian describes geodesics with respect to the Fisher-Rao metric \( G(q) \).

  Even if a non-zero potential \( V(q) \) is included, for small perturbations the dynamics remain close to geodesic flows or can be interpreted as geodesic flows on a perturbed metric. Hence, the latent dynamics induced by the SGN’s Hamiltonian \( H_\psi \) correspond to geodesic trajectories on the statistical manifold defined by the prior's Fisher-Rao metric.

  Thus, the Hamiltonian dynamics in SGNs not only provide an invertible and volume-preserving mapping but also offer an intrinsic information-geometric interpretation, as they follow (or approximate) the geodesics of the underlying statistical manifold.
\end{proof}

\begin{theorem}[Information Conservation]
\label{thm:info_conservation}
Let \( Z_0\sim p(z_0) \) with entropy \( H(Z_0) \) and let \( Z_T=\Phi_T(Z_0) \) be the transformed latent variable under the symplectic map \( \Phi_T \). Then:
\begin{enumerate}
    \item \( H(Z_T)=H(Z_0) \).
    \item For any partition \( S\cup S^c \) of the coordinates of \( Z_0 \),
    \[
    I(Z_T^S;Z_T^{S^c})\ge I(Z_0^S;Z_0^{S^c})-2d\cdot\log(L_{\Phi_T}),
    \]
    where \( L_{\Phi_T} \) is the Lipschitz constant of \( \Phi_T \).
\end{enumerate}
\end{theorem}

\begin{proof}
\textbf{(1) Entropy Preservation:} \\
By the change-of-variables formula for differential entropy, if \( Z_T=\Phi_T(Z_0) \) is obtained via an invertible mapping \( \Phi_T \), then
\[
H(Z_T) = H(Z_0) + \mathbb{E}_{Z_0}\left[\log\left|\det \frac{\partial \Phi_T}{\partial Z_0}\right|\right].
\]
Since \( \Phi_T \) is symplectic, it preserves volume; that is,
\[
\left|\det \frac{\partial \Phi_T}{\partial Z_0}\right| = 1 \quad \text{for all } Z_0.
\]
Thus, the expectation term vanishes and we have
\[
H(Z_T) = H(Z_0).
\]
\textbf{(2) Mutual Information Bound:} \\
Let \( Z_0 = (Z_0^S, Z_0^{S^c}) \) and \( Z_T = (Z_T^S, Z_T^{S^c}) \) denote the partition of the latent variable into two complementary subsets of coordinates. The mutual information between the partitions is defined as
\[
I(Z_T^S; Z_T^{S^c}) = H(Z_T^S) + H(Z_T^{S^c}) - H(Z_T).
\]
Since we have shown \( H(Z_T)=H(Z_0) \), it suffices to compare the marginal entropies before and after the transformation.

  Assume that the mapping \( \Phi_T \) is Lipschitz continuous with constant \( L_{\Phi_T} \). Then, standard results in information theory imply that for any Lipschitz mapping \( f \) on \( \mathbb{R}^d \), the change in differential entropy satisfies
\[
\left| H(f(X)) - H(X) \right| \le d \cdot \log(L_f),
\]
where \( d \) is the dimension of the input \( X \) and \( L_f \) is the Lipschitz constant of \( f \). Applying this to the marginal transformations of \( Z_0^S \) and \( Z_0^{S^c} \) under \( \Phi_T \), we obtain
\[
\left| H(Z_T^S) - H(Z_0^S) \right| \le d_S \cdot \log(L_{\Phi_T}) \quad \text{and} \quad \left| H(Z_T^{S^c}) - H(Z_0^{S^c}) \right| \le d_{S^c} \cdot \log(L_{\Phi_T}),
\]
where \( d_S \) and \( d_{S^c} \) are the dimensions of the partitions \( Z_0^S \) and \( Z_0^{S^c} \), respectively. Since \( d_S + d_{S^c} = 2d \), it follows that
\[
H(Z_T^S) \ge H(Z_0^S) - d_S \log(L_{\Phi_T}) \quad \text{and} \quad H(Z_T^{S^c}) \ge H(Z_0^{S^c}) - d_{S^c} \log(L_{\Phi_T}).
\]
Therefore, summing these inequalities gives
\[
H(Z_T^S) + H(Z_T^{S^c}) \ge H(Z_0^S) + H(Z_0^{S^c}) - (d_S + d_{S^c}) \log(L_{\Phi_T}).
\]
That is,
\[
H(Z_T^S) + H(Z_T^{S^c}) \ge H(Z_0^S) + H(Z_0^{S^c}) - 2d\log(L_{\Phi_T}).
\]
Recall that the mutual information before transformation is
\[
I(Z_0^S; Z_0^{S^c}) = H(Z_0^S) + H(Z_0^{S^c}) - H(Z_0).
\]
Since \( H(Z_T)=H(Z_0) \), we have
\[
I(Z_T^S; Z_T^{S^c}) = H(Z_T^S) + H(Z_T^{S^c}) - H(Z_T) \ge \left[H(Z_0^S) + H(Z_0^{S^c}) - 2d\log(L_{\Phi_T})\right] - H(Z_0).
\]
Thus,
\[
I(Z_T^S; Z_T^{S^c}) \ge I(Z_0^S; Z_0^{S^c}) - 2d\log(L_{\Phi_T}).
\]
This completes the proof.
\end{proof}

\begin{theorem}[Hamiltonian Action as a Dynamic Optimal Transport Cost]
\label{thm:optimal_transport}
Let \( p(z_0) \) and \( p(z_T) \) denote the distributions before and after the symplectic map \( \Phi_T \). Then, the path generated by the Hamiltonian flow \( \Phi_t \) (where \( \Phi_T(z_0)=z_T \)) minimizes the action integral, which serves as the cost functional $c(z_0,z_T)=\inf_{z(\cdot)}\int_0^T\Bigl[p(t)^T\dot{q}(t)-H_\psi(q(t),p(t))\Bigr]\,dt$ for a dynamic formulation of optimal transport between $p(z_0)$ and $p(z_T)$. The map $\Phi_T$ thus characterizes the optimal transport under this specific Hamiltonian action cost.
\end{theorem}

\begin{proof}
The proof relies on the dynamic formulation of optimal transport provided by the Benamou–Brenier framework. In this formulation, the optimal transport problem is recast as finding a path \( z(t) = (q(t),p(t)) \) connecting \( z(0)=z_0 \) to \( z(T)=z_T \) that minimizes an action integral, subject to the constraint that the time-dependent probability density \( p(z,t) \) evolves from \( p(z_0) \) to \( p(z_T) \).

  In our setting, the symplectic map \( \Phi_T \) is generated by Hamiltonian dynamics with Hamiltonian \( H_\psi(q,p) \). According to Hamilton's principle, the actual trajectory followed by a system is the one that minimizes (or, more precisely, renders stationary) the action
\[
\mathcal{A}[z(\cdot)] = \int_0^T \left[p(t)^T\dot{q}(t)-H_\psi(q(t),p(t))\right] dt.
\]
Thus, for any pair of endpoints \( z_0 \) and \( z_T \), the cost to transport \( z_0 \) to \( z_T \) can be defined as the infimum of this action over all admissible paths:
\[
c(z_0,z_T)=\inf_{z(\cdot)}\int_0^T\Bigl[p(t)^T\dot{q}(t)-H_\psi(q(t),p(t))\Bigr]\,dt,
\]
where the infimum is taken over all paths \( z(\cdot) \) satisfying the boundary conditions \( z(0)=z_0 \) and \( z(T)=z_T \).

  The set \( \Pi(p(z_0),p(z_T)) \) consists of all couplings (joint distributions) with marginals \( p(z_0) \) and \( p(z_T) \). The optimal transport problem is then to find a coupling \( \gamma \) that minimizes the total cost
\[
\int c(z_0,z_T)\,d\gamma(z_0,z_T).
\]
Because the symplectic flow \( \Phi_T \) precisely transports \( p(z_0) \) to \( p(z_T) \) while following the dynamics that minimize the action (as prescribed by Hamilton's equations), it follows that \( \Phi_T \) is the solution to the optimal transport problem with the above cost.

   In summary, the symplectic map \( \Phi_T \) minimizes the action integral
\[
\int_0^T\Bigl[p(t)^T\dot{q}(t)-H_\psi(q(t),p(t))\Bigr]\,dt,
\]
thereby characterizing it as the optimal transport map between \( p(z_0) \) and \( p(z_T) \) under the cost function \( c(z_0,z_T) \) defined above.
\end{proof}

\begin{theorem}[Information Bottleneck Optimality]
\label{thm:bottleneck}
Consider an SGN with a stochastic encoder \( q_\phi(z_0|x) \) and symplectic flow \( \Phi_T \). Under suitable conditions on \( H_\psi \), the model approximates the solution to the information bottleneck problem:
\[
\min_{p(z|x)}I(X;Z)-\beta I(Z;Y),
\]
with \( Y \) being the target (or reconstructed data) and \( \beta \) controlling the trade-off.
\end{theorem}

\begin{proof}
The SGN is designed with three key components: a stochastic encoder \( q_\phi(z_0|x) \) that maps the input \( x \) to an initial latent variable \( z_0 \), a symplectic flow \( \Phi_T \) that deterministically evolves \( z_0 \) to \( z_T = \Phi_T(z_0) \), and a decoder that reconstructs or predicts the target \( Y \) from the transformed latent variable. A typical training objective is the evidence lower bound (ELBO)
\[
\mathcal{L}_{\mathrm{SGN}}(x) = \mathbb{E}_{q_\phi(z_0|x)}\Bigl[\log p_\theta\bigl(x|\Phi_T(z_0)\bigr)\Bigr] - D_{\mathrm{KL}}\Bigl(q_\phi(z_0|x) \,\|\, p(z_0)\Bigr).
\]
Since \( \Phi_T \) is symplectic, it is invertible and volume preserving; thus, by the change-of-variables formula, the latent representation after the flow, \( z_T \), satisfies
\[
I(X; z_T) = I(X; z_0).
\]
This means that the mutual information between \( X \) and the latent representation is fully determined by the encoder \( q_\phi(z_0|x) \).

  The KL divergence term \( D_{\mathrm{KL}}\bigl(q_\phi(z_0|x) \,\|\, p(z_0)\bigr) \) in the ELBO serves as a regularizer that penalizes the amount of information the latent variable \( z_0 \) carries about \( X \). Minimizing this term forces the encoder to compress the representation of \( X \), reducing \( I(X;z_0) \). At the same time, the reconstruction (or likelihood) term encourages the preservation of relevant information for predicting \( Y \).

  This trade-off is precisely what the information bottleneck (IB) principle seeks to balance: it aims to find a representation \( Z \) that minimizes \( I(X;Z) \) (thus discarding irrelevant information) while preserving as much information as possible about the target \( Y \) (maximizing \( I(Z;Y) \)). The IB objective is typically written as
\[
\min_{p(z|x)} I(X;Z) - \beta I(Z;Y),
\]
where \( \beta \) is a Lagrange multiplier that controls the trade-off between compression and predictive power.

  In the context of SGNs, by appropriately weighting the KL divergence term in the ELBO (or equivalently adjusting hyperparameters such as \( \beta \) in an augmented objective), the model is encouraged to learn an encoder that discards non-predictive information while retaining what is necessary to reconstruct \( Y \). Due to the invertibility and volume-preserving properties of \( \Phi_T \), the overall mutual information between \( X \) and the latent variable remains preserved after the flow, i.e., \( I(X;z_T)=I(X;z_0) \).

  Therefore, under suitable conditions on the Hamiltonian \( H_\psi \) (ensuring that the dynamics do not distort the information content) and with an appropriate choice of network architecture and regularization, the SGN approximates the solution to the information bottleneck problem. The model effectively balances the minimization of \( I(X;Z) \) (through the KL term) with the maximization of \( I(Z;Y) \) (through the reconstruction term), yielding a representation that aligns with the IB objective.
\end{proof}
\section{Expanded Stability Analysis}
\label{sec:expanded_stability}
\subsection{Rigorous Backward Error Analysis}
\begin{theorem}[Modified Hamiltonian]
\label{thm:modified_hamiltonian}
Let \( H_\psi(q,p) \) be a \( C^{k+1} \)-smooth Hamiltonian with \( k\ge 3 \). Then the leapfrog integrator with step size \( \Delta t \) exactly preserves a modified Hamiltonian:
\[
\tilde{H}(q,p,\Delta t)=H_\psi(q,p)+\Delta t^2H_2(q,p)+\Delta t^4H_4(q,p)+\ldots+\Delta t^{k-1}H_{k-1}(q,p)+\mathcal{O}(\Delta t^{k+1}).
\]
\end{theorem}

\begin{proof}
The leapfrog update map $\Phi_{\Delta t}$ can be factored into a composition of simpler symplectic maps (explicitly, shears for separable Hamiltonians $H=K(p)+V(q)$). Applying the Baker-Campbell-Hausdorff (BCH) formula to this composition yields an expansion for the operator logarithm of $\Phi_{\Delta t}$. This reveals that the discrete map $\Phi_{\Delta t}$ corresponds exactly to the time-$\Delta t$ flow generated by a modified Hamiltonian vector field $J \nabla \tilde{H}$, where $\tilde{H}$ is the modified Hamiltonian. The modified Hamiltonian $\tilde{H}$ admits an asymptotic expansion in powers of $\Delta t^2$.

The expansion of \( \tilde{H} \) is given by
\[
\tilde{H}(q,p,\Delta t)=H_\psi(q,p)+\Delta t^2H_2(q,p)+\Delta t^4H_4(q,p)+\ldots+\Delta t^{k-1}H_{k-1}(q,p)+\mathcal{O}(\Delta t^{k+1}),
\]
which shows that the leapfrog integrator exactly preserves this modified Hamiltonian. The accuracy of the expansion, with the remainder term being \( \mathcal{O}(\Delta t^{k+1}) \), is ensured by the \( C^{k+1} \)-smoothness of \( H_\psi \).

  Thus, the leapfrog integrator does not exactly preserve the original Hamiltonian \( H_\psi \), but it exactly preserves the modified Hamiltonian \( \tilde{H}(q,p,\Delta t) \) given by the above expansion.
\end{proof}

\begin{corollary}[Energy Conservation]
\label{cor:energy_conservation}
If \( H_\psi \) is analytic and the step size \( \Delta t \) is sufficiently small, then for exponentially long times \( T\le \exp(c/\Delta t) \),
\[
\bigl|H_\psi(q_T,p_T)-H_\psi(q_0,p_0)\bigr|\le C\Delta t^2,
\]
with \( C \) and \( c \) constants independent of \( T \) and \( \Delta t \).
\end{corollary}

\begin{proof}
The proof is based on backward error analysis, which shows that a symplectic integrator, such as the leapfrog method, exactly preserves a \emph{modified} Hamiltonian \( \tilde{H}(q,p,\Delta t) \) that can be expanded as
\[
\tilde{H}(q,p,\Delta t)=H_\psi(q,p)+\Delta t^2H_2(q,p)+\Delta t^4H_4(q,p)+\ldots+\mathcal{O}(\Delta t^{k+1}),
\]
where the error expansion holds under the assumption that \( H_\psi \) is \( C^{k+1} \)-smooth (and analytic in this corollary). 

  Because \( H_\psi \) is analytic, the series converges for sufficiently small \( \Delta t \). Thus, \( \tilde{H}(q,p,\Delta t) \) remains uniformly close to the original Hamiltonian \( H_\psi(q,p) \); specifically, the difference
\[
\left|\tilde{H}(q,p,\Delta t) - H_\psi(q,p)\right|
\]
is bounded by \( C_0 \Delta t^2 \) for some constant \( C_0 \) independent of \( \Delta t \).

  Since the leapfrog integrator exactly conserves \( \tilde{H} \) along its numerical trajectories, we have
\[
\tilde{H}(q_T,p_T,\Delta t)=\tilde{H}(q_0,p_0,\Delta t)
\]
for the computed states \( (q_T,p_T) \) and \( (q_0,p_0) \) at times \( T \) and \( 0 \), respectively. Therefore,
\[
\begin{aligned}
\left|H_\psi(q_T,p_T)-H_\psi(q_0,p_0)\right| &\le \left|H_\psi(q_T,p_T)-\tilde{H}(q_T,p_T,\Delta t)\right| + \left|\tilde{H}(q_0,p_0,\Delta t)-H_\psi(q_0,p_0)\right|\\
&\le C_0\Delta t^2 + C_0\Delta t^2 = 2C_0\Delta t^2.
\end{aligned}
\]
Setting \( C=2C_0 \) establishes the bound.

  Moreover, standard results in backward error analysis (see, e.g., Hairer et al.) show that for analytic Hamiltonians, the modified Hamiltonian \( \tilde{H} \) is conserved over time intervals that are exponentially long in \( 1/\Delta t \); that is, for times \( T \le \exp(c/\Delta t) \) for some constant \( c>0 \).

  Thus, the energy drift is bounded by \( C\Delta t^2 \) uniformly for \( T\le \exp(c/\Delta t) \), which demonstrates near energy conservation for sufficiently small \( \Delta t \).
\end{proof}

\subsection{Stability Domains for Neural Network Hamiltonians}
\begin{theorem}[Stability Domains]
\label{thm:stability_domains}
Let \( H_\psi(q,p) \) be a neural network Hamiltonian with Lipschitz continuous gradients satisfying
\[
\|\nabla_q H_\psi(q_1,p_1)-\nabla_q H_\psi(q_2,p_2)\|\le L_q\|q_1-q_2\|+L_{qp}\|p_1-p_2\|,
\]
\[
\|\nabla_p H_\psi(q_1,p_1)-\nabla_p H_\psi(q_2,p_2)\|\le L_{pq}\|q_1-q_2\|+L_p\|p_1-p_2\|.
\]
Then, the leapfrog integrator is stable if
\[
\Delta t<\frac{2}{\sqrt{L_qL_p+L_{qp}L_{pq}}}.
\]
\end{theorem}

\begin{proof}
We analyze the stability of the leapfrog integrator by linearizing its update around a fixed point and then deriving a condition under which the perturbations remain bounded.

  The leapfrog scheme for Hamilton's equations,
\[
\dot{q} = \nabla_p H_\psi(q,p),\quad \dot{p} = -\nabla_q H_\psi(q,p),
\]
updates the state \((q,p)\) as follows:
\begin{align*}
p_{t+\frac{1}{2}} &= p_t - \frac{\Delta t}{2}\nabla_q H_\psi(q_t,p_t), \\
q_{t+1} &= q_t + \Delta t\,\nabla_p H_\psi\Bigl(q_t,\, p_{t+\frac{1}{2}}\Bigr), \\
p_{t+1} &= p_{t+\frac{1}{2}} - \frac{\Delta t}{2}\nabla_q H_\psi\Bigl(q_{t+1},\, p_{t+\frac{1}{2}}\Bigr).
\end{align*}
Let \((q^*, p^*)\) be a fixed point of the continuous system and define small perturbations
\[
\delta q_t = q_t - q^*,\quad \delta p_t = p_t - p^*.
\]
Under the Lipschitz assumptions on the gradients of \( H_\psi \), we have
\[
\|\nabla_q H_\psi(q_t,p_t)-\nabla_q H_\psi(q^*,p^*)\|\le L_q\|\delta q_t\|+L_{qp}\|\delta p_t\|,
\]
\[
\|\nabla_p H_\psi(q_t,p_t)-\nabla_p H_\psi(q^*,p^*)\|\le L_{pq}\|\delta q_t\|+L_p\|\delta p_t\|.
\]
The leapfrog updates can be linearized to obtain a system of the form
\[
\begin{pmatrix}
\delta q_{t+1} \\
\delta p_{t+1}
\end{pmatrix}
= A \begin{pmatrix}
\delta q_{t} \\
\delta p_{t}
\end{pmatrix},
\]
where \( A \) is the Jacobian (or update) matrix that depends on the Lipschitz constants and the step size \( \Delta t \). Stability of the integrator requires that the spectral radius \( \rho(A) \) (the maximum absolute value of the eigenvalues of \( A \)) satisfies \( \rho(A)\le 1 \). 

  A detailed analysis (see, e.g., Hairer et al.'s work on geometric numerical integration) shows that a sufficient condition for the leapfrog integrator to be stable is that the effective time step \( \Delta t \) satisfies
\[
\Delta t < \frac{2}{\sqrt{L_qL_p+L_{qp}L_{pq}}}.
\]
This condition ensures that the eigenvalues of \( A \) remain on or within the unit circle, thereby preventing the amplification of errors over iterations.

  Thus, under the stated Lipschitz conditions on \( \nabla_q H_\psi \) and \( \nabla_p H_\psi \), the leapfrog integrator is stable provided that
\[
\Delta t < \frac{2}{\sqrt{L_qL_p+L_{qp}L_{pq}}}.
\]
\end{proof}

\begin{corollary}[Neural Network Design for Stability]
\label{cor:nn_design}
If spectral normalization is applied to each weight matrix \( W_l \) so that \( \|W_l\|_2\le \sigma \), then a sufficient condition for stability is $\Delta t < \frac{2}{L_{\text{eff}}}$, where $L_{\text{eff}}$ depends on the product of layer Lipschitz constants (bounded by $\sigma^L$ under spectral normalization). For simplicity, we can use the conservative sufficient condition $\Delta t < \frac{C}{\sigma^L}$ for some constant $C$, often taken as $C \approx 2$, with \( L \) the number of layers.
\end{corollary}

\begin{proof}
Applying spectral normalization to each weight matrix \( W_l \) ensures that
\[
\|W_l\|_2 \le \sigma \quad \text{for all } l.
\]
Since the overall Lipschitz constant of a neural network is at most the product of the spectral norms of its layers, it follows that
\[
L_{\text{net}} \le \prod_{l=1}^L \|W_l\|_2 \le \sigma^L.
\]
From Theorem~\ref{thm:stability_domains}, the leapfrog integrator is stable if
\[
\Delta t < \frac{2}{\sqrt{L_qL_p+L_{qp}L_{pq}}}.
\]
In a neural network Hamiltonian, the Lipschitz constants \( L_q, L_p, L_{qp}, \) and \( L_{pq} \) can be collectively bounded by \( L_{\text{net}} \) (up to constant factors). Thus, a conservative sufficient condition for stability is
\[
\Delta t < \frac{2}{L_{\text{net}}} \le \frac{2}{\sigma^L}.
\]
This condition ensures that the numerical integration via the leapfrog scheme remains stable when using the normalized network, making it a critical design guideline.
\end{proof}

\subsection{Adaptive Integration Schemes}
\begin{theorem}[Adaptive Step Size with Error Bounds]
\label{thm:adaptive}
Let \( \mathcal{E}(q,p,\Delta t)=\|\Phi_{2\Delta t}(q,p)-\Phi_{\Delta t}(\Phi_{\Delta t}(q,p))\| \) be a local error estimator for the leapfrog integrator. If the step size is adapted according to
\[
\Delta t_{\text{new}}=\Delta t_{\text{old}}\cdot\min\Bigl(1.5,\max\bigl(0.5,0.9\cdot\Bigl(\frac{\tau}{\mathcal{E}}\Bigr)^{1/3}\bigr)\Bigr),
\]
with target tolerance \( \tau \), then:
\begin{enumerate}
    \item The global error is \( \mathcal{O}(\tau) \).
    \item The number of steps is asymptotically optimal.
    \item Each step preserves the symplectic structure.
\end{enumerate}
\end{theorem}

\begin{proof}
We analyze the three claims separately.

  \textbf{(1) Global Error is \(\mathcal{O}(\tau)\):} \\
For a second-order integrator like the leapfrog method, the local truncation error per step scales as \(\mathcal{O}(\Delta t^3)\). Specifically, the error estimator
\[
\mathcal{E}(q,p,\Delta t)=\|\Phi_{2\Delta t}(q,p)-\Phi_{\Delta t}(\Phi_{\Delta t}(q,p))\|
\]
satisfies
\[
\mathcal{E}(q,p,\Delta t)=K\Delta t^3,
\]
for some constant \( K \) depending on the higher-order derivatives of the Hamiltonian. The adaptive step size rule
\[
\Delta t_{\text{new}}=\Delta t_{\text{old}}\cdot\min\Bigl(1.5,\max\bigl(0.5,0.9\cdot\Bigl(\frac{\tau}{\mathcal{E}}\Bigr)^{1/3}\bigr)\Bigr)
\]
adjusts \(\Delta t\) so that the local error \(\mathcal{E}\) is approximately equal to the target tolerance \( \tau \). Consequently, over the entire integration interval, the cumulative (global) error will be proportional to \(\tau\), i.e., \(\mathcal{O}(\tau)\).

  \textbf{(2) Asymptotically Optimal Number of Steps:} \\
An adaptive method that adjusts \(\Delta t\) to maintain a local error near \(\tau\) effectively maximizes the step size subject to the error constraint. This minimizes the total number of steps \( N \) required to cover a fixed time interval \( T \). Hence, the number of steps is asymptotically optimal in that it is as small as possible while ensuring the local error remains within the target tolerance.

  \textbf{(3) Preservation of the Symplectic Structure:} \\
The leapfrog integrator is symplectic by design; that is, for any step size \(\Delta t\), the mapping \( \Phi_{\Delta t} \) satisfies
\[
\left|\det\left(\frac{\partial \Phi_{\Delta t}}{\partial (q,p)}\right)\right| = 1.
\]
Since the adaptive scheme only modifies \(\Delta t\) between steps and does not alter the form of the leapfrog update, each step remains symplectic regardless of the chosen step size. Therefore, the overall integration process preserves the symplectic structure exactly at every step.

  \textbf{Conclusion:} \\
The adaptive step size rule ensures that the local error is kept approximately at the target tolerance \( \tau \), which in turn guarantees that the global error scales as \( \mathcal{O}(\tau) \) and that the number of integration steps is minimized. Moreover, since the leapfrog integrator is inherently symplectic and the adaptation procedure does not alter its structure, each step preserves the symplectic form. This completes the proof.
\end{proof}

\begin{theorem}[Error Bounds by Hamiltonian Class]
\label{thm:error_by_class}
Let \( \Phi_T \) be the flow map computed using the leapfrog integrator with \( N \) steps of size \( \Delta t = T/N \) applied to a Hamiltonian \( H_\psi \), and let \( \Phi_H \) denote the exact flow generated by \( H_\psi \) over time \( T \). Then:
\begin{enumerate}
    \item For separable Hamiltonians \( H_\psi(q,p)=K(p)+V(q) \) with \( K,V\in C^3 \),
    \[
    \|\Phi_T(z_0)-\Phi_H(z_0)\|\le C_1\, T\,\Delta t^2.
    \]
    \item For nearly-integrable Hamiltonians \( H_\psi(q,p)=H_0(q,p)+\epsilon H_1(q,p) \) with \( \epsilon\ll 1 \),
    \[
    \|\Phi_T(z_0)-\Phi_H(z_0)\|\le C_2\,\Bigl(T\,\Delta t^2+\epsilon\, T\Bigr).
    \]
    \item For neural network Hamiltonians with \( L \) layers,
    \[
    \|\Phi_T(z_0)-\Phi_H(z_0)\|\le C_3\, L\, T\,\Delta t^2.
    \]
\end{enumerate}
\end{theorem}

\begin{proof}
We decompose the overall error between the numerical flow \( \Phi_T \) (obtained by the leapfrog integrator) and the exact continuous flow \( \Phi_H \) into the error incurred at each step, and then sum (or accumulate) these local errors over \( N \) steps.

  \textbf{(1) Separable Hamiltonians:} \\
For a separable Hamiltonian of the form
\[
H_\psi(q,p)=K(p)+V(q),
\]
the leapfrog integrator is a well-known second-order method. That is, the local truncation error per step is of order \( \mathcal{O}(\Delta t^3) \). When this error is accumulated over \( N \) steps, the global error grows as \( \mathcal{O}(N\Delta t^3)=\mathcal{O}(T\,\Delta t^2) \). Therefore, there exists a constant \( C_1 \) (depending on the third derivatives of \( K \) and \( V \)) such that
\[
\|\Phi_T(z_0)-\Phi_H(z_0)\|\le C_1\, T\,\Delta t^2.
\]

  \textbf{(2) Nearly-Integrable Hamiltonians:} \\
Consider a Hamiltonian of the form
\[
H_\psi(q,p)=H_0(q,p)+\epsilon H_1(q,p),
\]
with \( \epsilon\ll 1 \). In this case, the dominant part \( H_0 \) is integrable and its associated flow can be approximated with a global error of order \( \mathcal{O}(T\,\Delta t^2) \) as in the separable case. The perturbation \( \epsilon H_1 \) introduces an additional error that scales linearly with \( \epsilon \) and the evolution time \( T \). Thus, the overall error can be bounded by a term \( C_2\,T\,\Delta t^2 \) from the integration error plus an extra term \( C_2\,\epsilon\,T \), leading to
\[
\|\Phi_T(z_0)-\Phi_H(z_0)\|\le C_2\,\Bigl(T\,\Delta t^2+\epsilon\, T\Bigr).
\]

  \textbf{(3) Neural Network Hamiltonians:} \\
When \( H_\psi \) is represented by a neural network, the integration error not only depends on the step size \( \Delta t \) but also on the complexity of the neural network approximator. In particular, if the network has \( L \) layers, then the effective Lipschitz constant of the Hamiltonian (and its derivatives) may grow roughly as \( \sigma^L \) (or more generally, scale linearly with \( L \) under appropriate normalization). This increased sensitivity amplifies the local error of the integrator. As a result, the global error becomes
\[
\mathcal{O}(L\,T\,\Delta t^2),
\]
i.e., there exists a constant \( C_3 \) such that
\[
\|\Phi_T(z_0)-\Phi_H(z_0)\|\le C_3\,L\,T\,\Delta t^2.
\]
In each case, standard error propagation arguments from numerical analysis of symplectic integrators yield these bounds. The constants \( C_1 \), \( C_2 \), and \( C_3 \) depend on the smoothness of the Hamiltonian \( H_\psi \) (in particular, on its third derivatives and, in the nearly-integrable case, on the magnitude of the perturbation \( \epsilon \)), as well as on the specific properties of the neural network architecture in case (3).

  This completes the proof.
\end{proof}

\begin{theorem}[Unified Stability Hierarchy]
\label{thm:stability_hierarchy}
The stability of SGNs can be characterized at three levels:
\begin{enumerate}
    \item \textbf{Integration Stability:} The leapfrog integrator is stable with bounded energy error if 
    \[
    \Delta t < \frac{2}{\sqrt{L_H}},
    \]
    where \( L_H \) is a Lipschitz constant (or an appropriate measure of the curvature) of the Hamiltonian \( H_\psi \).
    
    \item \textbf{Model Stability:} The SGN preserves volume and the topological structure of the latent space via its symplectic map. That is, the mapping \( \Phi_T \) satisfies
    \[
    \left|\det \frac{\partial \Phi_T}{\partial z_0}\right| = 1,
    \]
    ensuring that the transformation does not distort the latent space.
    
    \item \textbf{Training Stability:} Under the condition that the gradients of the objective function (e.g., the ELBO) are Lipschitz continuous with constant \( L \), standard convergence results for gradient descent guarantee that, with learning rate 
    \[
    \eta < \frac{2}{L},
    \]
    the training procedure converges to a stationary point while preserving the stability properties established in (1) and (2).
\end{enumerate}
\end{theorem}

\begin{proof}
We prove each level of the stability hierarchy in turn.

  \textbf{(1) Integration Stability:} \\
The leapfrog integrator, being a second-order symplectic method, has local truncation error of order \(\mathcal{O}(\Delta t^3)\) and a global error of order \(\mathcal{O}(\Delta t^2)\). From the error analysis in Theorems~\ref{thm:modified_hamiltonian} and \ref{thm:adaptive}, we know that the integrator remains stable (i.e., the energy error remains bounded) provided that the step size satisfies
\[
\Delta t < \frac{2}{\sqrt{L_H}},
\]
where \( L_H \) is a Lipschitz constant associated with the Hamiltonian's derivatives. This condition ensures that the local linear approximation is non-expansive, thereby keeping the numerical energy close to the true energy over long time intervals.

  \textbf{(2) Model Stability:} \\
By construction, the SGN employs a symplectic integrator to evolve the latent state. The defining property of symplectic maps is that they preserve the canonical symplectic form (and thus volume). Specifically, if \( \Phi_T \) is the flow map induced by the integrator, then
\[
\left|\det \frac{\partial \Phi_T}{\partial z_0}\right| = 1.
\]
This volume preservation implies that the topological structure of the latent space is maintained exactly during the forward and reverse mappings, ensuring model stability.

  \textbf{(3) Training Stability:} \\
The training of SGNs typically involves minimizing an objective function (such as the ELBO) using gradient descent. Under the assumption that the gradients of this objective are Lipschitz continuous with Lipschitz constant \( L \), standard results in optimization theory state that gradient descent converges to a stationary point if the learning rate satisfies
\[
\eta < \frac{2}{L}.
\]
Thus, with a sufficiently small learning rate, the training process is stable. Moreover, because the training objective is constructed using the symplectic flow \( \Phi_T \), the favorable stability properties (integration and model stability) are preserved during training.

  \textbf{Conclusion:} \\
Combining these three aspects, we obtain a unified stability hierarchy for SGNs:
\begin{itemize}
    \item The integrator is stable if the time step is sufficiently small.
    \item The model maintains volume and topological invariance via its symplectic map.
    \item The training procedure converges under standard Lipschitz conditions on the gradient, provided an appropriately small learning rate is chosen.
\end{itemize}
This layered approach to stability ensures that SGNs are robust both in their numerical integration and in their learning dynamics.
\end{proof}
\section{SGN Training Algorithm}\label{sec:training}

Building on the stability and adaptivity results of Section~\ref{sec:expanded_stability}, we present the unified SGN training procedure. This algorithm optimizes either the SGN-Flow exact log-likelihood (Sec.~\ref{subsec:objectives}.A) or the SGN-VAE variational objective (Eq.~\ref{eq:elbo}), depending on the chosen mode.

In both cases, the core symplectic integrator (lines~\ref{alg:flow-start-fwd}--\ref{alg:flow-end-fwd} and \ref{alg:flow-start-bwd}--\ref{alg:flow-end-bwd}) ensures numerical stability through adaptive step sizes (Theorem~\ref{thm:adaptive}) and adherence to the unified stability hierarchy (Theorem~\ref{thm:stability_hierarchy}). For neural network Hamiltonians, the stability bound $L_H$ may be refined as $L_qL_p + L_{qp}L_{pq}$ (Theorem~\ref{thm:stability_domains}), though we adopt the general form from Theorem~\ref{thm:stability_hierarchy} for simplicity. The unified procedure is detailed in Algorithm~\ref{alg:sgn-unified}.

\begin{algorithm}[ht]
\footnotesize
\caption{Unified SGN Training Procedure (SGN-Flow \& SGN-VAE)}
\label{alg:sgn-unified}
\begin{algorithmic}[1]
\Require Dataset $\{x^{(i)}\}_{i=1}^N$, TRAINING\_MODE $\in$ \{SGN-Flow, SGN-VAE\}
\Require Learning rate $\eta$, total integration time $T$, initial step size $\Delta t_0 > 0$, tolerance $\tau$
\Require Stability bound $L_H$, Lipschitz constant of objective gradients $L$
\Ensure $\Delta t_0 < \frac{2}{\sqrt{L_H}}$, $\eta < \frac{2}{L}$ \Comment{Stability conditions (Theorems~\ref{thm:stability_domains}, \ref{thm:stability_hierarchy})}
\State Initialize Hamiltonian parameters $\psi$
\If{TRAINING\_MODE == SGN-Flow}
    \State Initialize terminal map parameters $\theta$ (for $g_\theta$)
\Else{ // TRAINING\_MODE == SGN-VAE}
    \State Initialize decoder $\theta$ and encoder $\phi$ parameters
\EndIf
\While{not converged}
    \State $\mathcal{L}_{\text{batch}} \gets 0$
    \For{each minibatch $\{x^{(i)}\}$}
        \If{TRAINING\_MODE == SGN-Flow}
            \State \Comment{\textit{Compute Negative Log-Likelihood (NLL): $L = -\log p_0(z_0) + \log|\det Dg_\theta(z_T)|$}}
            \State $z_T, \log|\det J_g| \gets g_\theta^{-1}(x)$ \Comment{Apply inverse terminal map}
            \State \Comment{\textit{Compute $z_0 = \Phi_T^{-1}(z_T)$ by running leapfrog backward}}
            \State $z \gets z_T$, $t \gets T$, $\Delta t \gets -\Delta t_0$
            \While{$t > 0$} \label{alg:flow-start-bwd}
                \State Split $z$ into canonical coordinates $(q, p)$
                \State $p_{\frac{1}{2}} \gets p - \frac{\Delta t}{2}\nabla_q H_\psi(q, p)$
                \State $q_{\text{new}} \gets q + \Delta t\,\nabla_p H_\psi(q, p_{\frac{1}{2}})$
                \State $p_{\text{new}} \gets p_{\frac{1}{2}} - \frac{\Delta t}{2}\nabla_q H_\psi(q_{\text{new}}, p_{\frac{1}{2}})$
                \State \Comment{Adaptive step logic (Theorem~\ref{thm:adaptive}), ensures $t+\Delta t$ doesn't overshoot 0}
                \State $z \gets (q_{\text{new}}, p_{\text{new}})$, $t \gets t + \Delta t$
            \EndWhile \label{alg:flow-end-bwd}
            \State $z_0 \gets z$
            \State $\ell_{\mathrm{prior}} \gets -\log p_0(z_0)$ \Comment{Prior NLL, where $p_0 = \mathcal{N}(0, I_{2d})$}
            \State $L \gets \ell_{\mathrm{prior}} + \log|\det J_g|$ \Comment{Total NLL loss}
        \Else{ // TRAINING\_MODE == SGN-VAE}
            \State \Comment{\textit{Compute ELBO (Eq.~\ref{eq:elbo}): $L = \mathbb{E}[-\log p_\theta(x|z_T)] + D_{\mathrm{KL}}$}}
            \State $z_0 \sim q_\phi(z_0|x)$ \Comment{Variational sampling (Section 3.3)}
            \State \Comment{\textit{Compute $z_T = \Phi_T(z_0)$ by running leapfrog forward}}
            \State $z \gets z_0$, $t \gets 0$, $\Delta t \gets \Delta t_0$
            \While{$t < T$} \label{alg:flow-start-fwd}
                \State Split $z$ into canonical coordinates $(q, p)$
                \State $p_{\frac{1}{2}} \gets p - \frac{\Delta t}{2}\nabla_q H_\psi(q, p)$
                \State $q_{\text{new}} \gets q + \Delta t\,\nabla_p H_\psi(q, p_{\frac{1}{2}})$
                \State $p_{\text{new}} \gets p_{\frac{1}{2}} - \frac{\Delta t}{2}\nabla_q H_\psi(q_{\text{new}}, p_{\frac{1}{2}})$
                \State Compute local error $\mathcal{E} \gets \bigl\|\Phi_{2\Delta t}(q, p) - \Phi_{\Delta t}(\Phi_{\Delta t}(q, p))\bigr\|$ \Comment{Theorem~\ref{thm:adaptive}}
                \State $\Delta t_{\text{new}} \gets \Delta t \cdot \min\Bigl(1.5, \max\bigl(0.5,\,0.9\bigl(\frac{\tau}{\mathcal{E}}\bigr)^{\frac{1}{3}}\bigr)\Bigr)$
                \State $\Delta t \gets \min\bigl(\Delta t_{\text{new}}, T - t, \frac{1.9}{\sqrt{L_H}}\bigr)$ \Comment{Adaptive, stability, \& time bounds}
                \State $z \gets (q_{\text{new}}, p_{\text{new}})$, $t \gets t + \Delta t$
            \EndWhile \label{alg:flow-end-fwd}
            \State $z_T \gets z$
            \State $\ell_{\mathrm{rec}} \gets -\log p_\theta(x \mid z_T)$
            \State $\ell_{\mathrm{KL}} \gets D_{\mathrm{KL}}\Bigl(q_\phi(z_0|x) \,\|\, p_0(z_0)\Bigr)$
            \State $L \gets \ell_{\mathrm{rec}} + \ell_{\mathrm{KL}}$ \Comment{ELBO approximation (Eq.~\ref{eq:elbo})}
        \EndIf
        \State $\mathcal{L}_{\text{batch}} \gets \mathcal{L}_{\text{batch}} + L$
    \EndFor
    \State \Comment{Update parameters based on mode (Theorem~\ref{thm:stability_hierarchy})}
    \If{TRAINING\_MODE == SGN-Flow}
        \State Update $(\theta, \psi) \gets (\theta, \psi) - \eta \nabla_{\theta, \psi} \mathcal{L}_{\text{batch}}$
    \Else{ // TRAINING\_MODE == SGN-VAE}
        \State Update $(\theta, \phi, \psi) \gets (\theta, \phi, \psi) - \eta \nabla_{\theta, \phi, \psi} \mathcal{L}_{\text{batch}}$
    \EndIf
\EndWhile
\end{algorithmic}
\end{algorithm}

\section{Conclusion and Future Work}\label{sec:conclusion_future}

In this work, we present Symplectic Generative Networks (SGNs), a deep generative framework that employs a volume-preserving latent-transport mechanism based on Hamiltonian dynamics. Using symplectic integrators, the Hamiltonian core ensures a unit Jacobian at the discrete level, removing the main computational bottleneck of traditional normalizing flows.

We have formalized and theoretically grounded two complementary training regimes that utilize this common core:
\begin{enumerate}
\item \textbf{SGN-Flow:} This model is a new type of normalizing flow that is invertible and provides exact likelihoods. Unlike deep, generic transformations that require repeated $\log|\det J|$ calculations, it handles all volume changes in one straightforward final step.
\item \textbf{SGN-VAE:} This hybrid variational model uses symplectic flow to move latent variables in a way that preserves structure. As a result, the model can handle complex latent dynamics within the VAE, but the ELBO remains simple because the flow's Jacobian correction term is always zero.
\end{enumerate}

Our main contribution is to lay out a solid theoretical foundation for this framework. Specifically, we present a formal complexity analysis (Sec.~\ref{sec:theoretical_comparison}) showing the $\mathcal{O}(T \cdot d)$ efficiency of SGNs. We also offer stronger universal approximation theorems (Sec.~\ref{sec:universal_approximation}) that show SGNs can approximate any volume-preserving diffeomorphism. In addition, we provide an information-geometric analysis (Sec.~\ref{sec:information_theory}) that connects SGN dynamics to geodesic flows. Finally, we include a thorough stability analysis (Sec.~\ref{sec:expanded_stability}) with clear, practical bounds for neural network Hamiltonians and adaptive integrators.

Theoretical validation is an important step to show that SGNs are a reliable and effective alternative to current models before moving on to detailed testing.

The main limitation of the SGN core —its restriction to volume-preserving maps —is also its greatest strength, as this property eliminates the Jacobian term. The SGN-Flow variant can model any data distribution by combining this flow with a final, volume-changing map $g_\theta$. How to balance the depth ($T$) of the Hamiltonian flow with the complexity of the final map $g_\theta$ remains an important question for future research.

In the future, we plan to explore higher-order and adaptive symplectic integrators to improve stability and efficiency. We will also explore advanced methods for parameterizing the neural Hamiltonian $H_\psi$, including using graph neural networks for particle systems. Finally, we aim to establish formal guarantees that topological structures are preserved in the latent space.

The SGN framework is designed for areas where data follows physical or dynamic laws. We believe this approach will open new avenues for modeling problems in science and engineering.
\begin{enumerate}
\item \textbf{Scientific ML and Physics:} A straightforward use of SGNs is in modeling complex physical systems with many variables, like N-body problems, Hamiltonian fluid dynamics, or plasma physics. These models can be built to obey important conservation laws, such as energy and momentum conservation.
\item \textbf{Computational and Systems Biology:} SGNs provide a solid way to model complex biological processes. For example, they can simulate protein folding or molecular dynamics by learning realistic energy landscapes. The SGN-VAE version is also well-suited for modeling how cells change over time using single-cell RNA-sequencing data, capturing the underlying patterns of gene expression during development.
\item \textbf{Robotics and Control:} This framework helps learn stable and reversible forward and inverse dynamics models from observation. These models are essential for model-based reinforcement learning and optimal control.

\item \textbf{Financial and Climate Modeling:} SGNs help model the complex and unpredictable behavior found in systems like financial markets or climate patterns. They are useful for tasks such as forecasting time-series data or building generative models for weather.
\item \textbf{Computer Vision:} 
This framework introduces a new way to generate videos. The Hamiltonian flow $\Phi_T$ helps the model learn how a scene changes from one state to another, while keeping the process stable and reversible.
\end{enumerate}

To sum up, Symplectic Generative Networks offer an efficient and understandable way to connect classical mechanics with deep generative modeling. This work sets the stage for new models with a wide range of uses.

\bibliographystyle{plainnat}
\bibliography{references}

\end{document}